\documentclass[submission,copyright,creativecommons]{eptcs}
\providecommand{\event}{} 
\usepackage{mathrsfs}
\usepackage{amsmath}
\usepackage{amssymb}

\usepackage[]{algorithm2e}

\usepackage{enumitem}

\usepackage{graphicx}
\usepackage{epstopdf}

\usepackage{pbox}

\DeclareGraphicsExtensions{.eps}

\usepackage{color}

\newcommand{\todo}[1]{#1}
\newcommand{\revision}[1]{#1}
\newcommand{\revise}[1]{#1}

\newcommand{\qed}{$\hfill\blacksquare$}

\newcommand{\comment}[1]{}

\newcounter{myctr}
\newenvironment{mylist}{\begin{list}{\arabic{myctr}.}
{\usecounter{myctr}
\setlength{\topsep}{1mm}\setlength{\itemsep}{0.25mm}
\setlength{\parsep}{0.1mm}
\setlength{\itemindent}{0mm}\setlength{\partopsep}{0mm}
\setlength{\labelwidth}{15mm}
\setlength{\leftmargin}{4mm}}}{\end{list}}

\newenvironment{myitemize}{\begin{list}{$\bullet$}
{\setlength{\topsep}{1mm}\setlength{\itemsep}{0.25mm}
\setlength{\parsep}{0.1mm}
\setlength{\itemindent}{0mm}\setlength{\partopsep}{0mm}
\setlength{\labelwidth}{15mm}
\setlength{\leftmargin}{4mm}}}{\end{list}}

\newcommand{\class}{\ensuremath{\mathcal{C}}\xspace}
\newcommand{\concept}{\ensuremath{\mathit{c}}\xspace}

\newcommand{\Spec}{\ensuremath{\Phi}\xspace}
\newcommand{\target}{\ensuremath{\mathit{c}}\xspace}
\newcommand{\domain}{\ensuremath{\mathbf{E}}\xspace}
\newcommand{\ex}{\ensuremath{\mathit{x}}\xspace}
\newcommand{\ExSet}{\ensuremath{\mathit{X}}\xspace}

\newtheorem{theorem}{Theorem}[section]
\newtheorem{lemma}[theorem]{Lemma}

\newtheorem{corollary}[theorem]{Corollary}
\newtheorem{example}{Example}[section]
\newtheorem{definition}{Definition}[section]

\newtheorem{langfamily}{Language Family}

\newenvironment{proof}[1][Proof]{\begin{trivlist}
\item[\hskip \labelsep {\bfseries #1}]}{\end{trivlist}}

\newcommand{\ogis}{\ensuremath{\mathsf{OGIS}}\xspace} 

\newcommand{\tcegis}{\mathtt{CEGIS}}
\newcommand{\infcegis}{\mathtt{CEGIS}}
\newcommand{\infmncegis}{\mathtt{MINCEGIS}}
\newcommand{\infbcegis}{\mathtt{CBCEGIS}}
\newcommand{\infhmncegis}{\mathtt{PBCEGIS}}
\newcommand{\cegis}{\mathtt{cegis}}
\newcommand{\mncegis}{\mathtt{mincegis}}
\newcommand{\hmncegis}{\mathtt{pbcegis}}
\newcommand{\bcegis}{\mathtt{cbcegis}}

\newcommand{\dseq}{\mathtt{\delta}}

\newcommand{\bound}{B}

\newcommand{\lang}{L}
\newcommand{\langfun}{\todo{L_{map}}}
\newcommand{\trace}{\tau}
\newcommand{\nat}{\mathbb{N}}
\newcommand{\natbot}{\ensuremath{\nat_{\bot}}}
\newcommand{\range}{range}
\newcommand{\samples}{\mathtt{SAMPLE}}
\newcommand{\algo}{P}
\newcommand{\template}{\mathtt{TEMPLATE}}
\newcommand{\verifier}{\mathtt{CHECK}}
\newcommand{\mnverifier}{\mathtt{MINCHECK}}
\newcommand{\hmnverifier}{\mathtt{PBCHECK}}
\newcommand{\bverifier}{\mathtt{CBCHECK}}

\newcommand{\ce}{\mathtt{cex}}
\newcommand{\mncemap}{\mathtt{mce}}
\newcommand{\hmncemap}{\mathtt{pce}}

\newcommand{\engine}{T}
\newcommand{\siml}{\mathtt{sim}}

\newcommand{\deflearner}{\mathtt{\mathbf{L}}}
\newcommand{\deforacle}{\mathtt{\mathbf{O}}}

\newcommand{\inflearner}{\mathtt{LEARN}}
\newcommand{\learner}{\mathtt{learn}}

\newcommand{\orint}{\ensuremath{\mathcal{O}}}
\newcommand{\FIS}{FIS\xspace}
\newcommand{\queries}{\mathtt{\mathbf{Q}}}
\newcommand{\responses}{\mathtt{\mathbf{R}}}
\newcommand{\dialogue}{\mathtt{\mathbf{D}}} 

\newcommand{\defquery}{\ensuremath{\mathcal{Q}}}
\newcommand{\defresponse}{\ensuremath{\mathcal{R}}}

\newcommand{\qmem}{\ensuremath{q_{\text{mem}}}}
\newcommand{\qposwit}{\ensuremath{q_{\text{wit}}^+}}
\newcommand{\qnegwit}{\ensuremath{q_{\text{wit}}^-}}
\newcommand{\qcorr}{\ensuremath{q_{\text{corr}}}}
\newcommand{\qeq}{\ensuremath{q_{\text{eq}}}}
\newcommand{\qsub}{\ensuremath{q_{\text{sub}}}}
\newcommand{\qdiff}{\ensuremath{q_{\text{diff}}}}
\newcommand{\qccorr}{\ensuremath{q_{\text{ccorr}}}}
\newcommand{\qceq}{\ensuremath{q_{\text{ce}}}}
\newcommand{\qcsub}{\ensuremath{q_{\text{csub}}}}

\newcommand{\satEx}{\vdash} 

\newcommand{\nex}{\mathtt{negex}}
\newcommand{\cemap}{\mathtt{ce}}
\newcommand{\negset}{\mathtt{negset}}

\newcommand{\scratch}{\mathtt{S}}
\newcommand{\Scratch}{\mathscr{S}}

\newcommand{\powset}[1]{\ensuremath{2^{#1}}}

\newcommand{\acta}[1]{{#1}}

\title{A Theory of Formal Synthesis via Inductive Learning}
\author{Susmit Jha
\institute{United Technologies Research Center, Berkeley}
\email{jhask@utrc.utc.com}
\and
Sanjit A. Seshia
\institute{EECS, UC Berkeley}
\email{\quad sseshia@eecs.berkeley.edu}
}
\def\titlerunning{Formal Inductive Synthesis}
\def\authorrunning{S. Jha \& S. A. Seshia}
\begin{document}

\maketitle

\begin{abstract}
Formal synthesis is the process of generating a program satisfying a high-level formal specification.
In recent times, effective formal synthesis methods have been proposed based on 
the use of inductive learning. We refer to this class of methods that learn programs from examples 
as formal inductive synthesis.
In this paper, we present a theoretical framework for formal inductive synthesis. We discuss how
formal inductive synthesis differs from traditional machine learning.
We then describe oracle-guided inductive synthesis (OGIS), a framework that captures a
family of synthesizers that operate by iteratively querying an oracle.
An instance of OGIS that has had much practical impact is counterexample-guided inductive synthesis (CEGIS).
We present a theoretical characterization of CEGIS for learning any program that computes a recursive language.
In particular, we analyze the relative power of CEGIS variants where the types of counterexamples generated
by the oracle varies. We also consider the impact of bounded versus unbounded memory available to the learning algorithm.
In the special case where the universe of candidate programs is finite, we relate the speed of convergence
to the notion of teaching dimension studied in machine learning theory.
Altogether, the results of the paper take a first step towards a theoretical foundation for
the emerging field of formal inductive synthesis.
\end{abstract}

\comment {
Counterexample-guided inductive synthesis ($\cegis$) is used to synthesize programs from a candidate space of programs.
The technique is guaranteed to terminate and synthesize the correct program if the space of candidate programs is finite.
But the technique may or may not terminate with the correct program if the candidate space of programs is infinite.
In this paper, we perform a theoretical analysis of counterexample-guided inductive synthesis technique.
We investigate whether the set of candidate spaces for which the correct program can be synthesized using $\cegis$ depends on the
counterexamples used in inductive synthesis, that is, whether there are
{\it good mistakes} which would increase the synthesis power.
We investigate whether the use of minimal counterexamples instead of arbitrary counterexamples
expands the set of candidate spaces of programs for which inductive synthesis can successfully synthesize a correct program.
 We consider two kinds of  counterexamples:  minimal counterexamples
 and {\it history bounded} counterexamples. The history bounded
 counterexample used in
 any iteration of $\cegis$ is
 bounded by the examples used in previous iterations of inductive synthesis.
 We examine the relative change in power of inductive
 synthesis in both cases. We show that the synthesis technique using
 minimal counterexamples $\mncegis$ has the
 same synthesis power as $\cegis$
 but the synthesis technique using history bounded counterexamples $\hmncegis$
 has different power than that of $\cegis$, but none dominates the other.
}

\section{Introduction}
\label{sec-intro} 
\label{sec:intro}

The field of formal methods has made enormous strides in recent decades.
Formal verification techniques such as model
checking~\cite{clarke-lop81,queille-sympprog82,clarke-00} and 
theorem proving (see, e.g.~\cite{owre-pvs-cade92,kaufmann-00,gordon-93})  
are used routinely in the computer-aided design of integrated circuits and
have been widely applied to find bugs in software, analyze models of embedded systems,
and find security vulnerabilities in programs and protocols. 
At the heart of many of these advances are computational reasoning engines such as
Boolean satisfiability (SAT) solvers~\cite{malik-cacm09}, 
Binary Decision Diagrams (BDDs)~\cite{bryant-ieeetc86}, and
satisfiability modulo theories (SMT) solvers~\cite{barrett-smtbookch09}.
Alongside these advances, there has been a growing interest in the
synthesis of programs or systems from formal specifications with
correctness guarantees. We refer 
to this area as {\em formal synthesis}. Starting with the seminal work
of Manna and Waldinger on deductive program synthesis~\cite{Manna80}
and Pnueli and Rosner on reactive synthesis from temporal
logic~\cite{pnueli-popl89}, there have been several advances that have
made formal synthesis practical in specific application domains such
as robotics, online education, and end-user programming.

%
%
Algorithmic approaches to formal synthesis range over a wide spectrum, from
{\em deductive synthesis} to {\em inductive synthesis}.
In deductive synthesis (e.g.,~\cite{Manna80}), a program is synthesized by 
constructively proving a theorem, employing logical inference and constraint solving.
On the other hand, inductive synthesis~\cite{gold67limit,summers-jacm77,shapiro1982algorithmic} 
seeks to find a program matching a set of input-output examples. At a
high level, it is thus an instance of
learning from examples, also termed as {\em inductive inference} or {\em machine learning}~\cite{angluin-acmcs83,Mitchell-ml}.
Many current approaches to synthesis blend induction and
deduction in the sense that even as they
generalize from examples, deductive procedures are used in the process
of generalization (see~\cite{seshia-dac12,Jha-thesis} for a detailed exposition). 
Even so, the term ``inductive synthesis'' is typically used to refer to all of them. We will
refer to these methods as {\em formal inductive synthesis} to
place an emphasis on correctness of the synthesized artifact.
These synthesizers generalize from examples by searching a restricted space of programs. 
In machine learning, this restricted space is called the {\em concept class}, and each element of that
space is often called a candidate {\em concept}. 
The concept class is usually specified syntactically. 
It has been recognized that this {\em syntax guidance}, also termed as a
{\em structure hypothesis}, can
be crucial in helping the synthesizer converge quickly to the target
concept~\cite{asplos06,seshia-dac12,alur-fmcad13}. 

The fields of formal inductive synthesis and 
machine learning have the same high-level goal: to
develop algorithmic techniques for 
{\em synthesizing a concept} (function, program, or classifier) 
{\em from observations} (examples, queries, etc.). 
However, there are also important differences in the problem
formulations and techniques used in both fields. We identify some of
the main differences below:
\begin{mylist}
\item {\em Concept Classes:} In traditional machine learning, the
classes of concepts to be synthesized tend to be specialized,
such as linear functions or half-spaces~\cite{regression}, convex
polytopes~\cite{Hegedus-convex}, 
neural networks of specific forms~\cite{deeplearning-ref},
Boolean formulas in fixed, bounded syntactic forms~\cite{Jackson-dnf}, 
and decision trees~\cite{Quinlan-decisiontrees}.
However, in formal synthesis, the target concepts are 
general programs or automata with constraints or finite bounds imposed
mainly to ensure tractability of synthesis.

\item {\em Learning Algorithms:} 
In traditional machine learning, just as concept classes tend to be
specialized, so also are the learning algorithms for those classes~\cite{Mitchell-ml}.
In contrast, in formal inductive synthesis, the trend is towards using general-purpose
decision procedures such as SAT solvers, SMT solvers, and model
checkers that are not specifically designed for inductive learning. 

\item {\em Exact vs. Approximate Learning:} 
In formal inductive synthesis, there is a strong emphasis on {\em exactly}
learning the target concept; i.e., the learner seeks to find a concept
that is consistent with all positive examples but not with any
negative example. The labels for examples are typically assumed to be
correct. Moreover, the learned concept should satisfy a formal specification. 
In contrast, the emphasis in traditional machine
learning is on techniques that perform {\em approximate} learning,
where input data can be noisy, some amount of misclassification can be
tolerated, there is no formal specification, and the overall goal is to optimize
a cost function (e.g., capturing classification error).

\item {\em Emphasis on Oracle-Guidance:} 
In formal inductive synthesis, there is a big emphasis on learning in
the presence of an oracle, which is typically implemented using a
general-purpose decision procedure or sometimes even a human user.
Moreover, and importantly, the design of this oracle is
part of the design of the synthesizer.
In contrast, in traditional machine learning, the use of oracles is rare, and instead
the learner typically selects examples from a corpus, often drawing
examples independently from an underlying probability distribution. Even when
oracles are used, they are assumed to be black boxes that the learner has
no control over. The oracle is part of the problem definition in machine learning,
whereas in formal inductive synthesis, the design of the oracle is 
part of the solution.

\end{mylist}
The last item, oracle-guidance, is a particularly important difference, and
informs the framework we proposed in this paper.

In this paper, we take first steps towards a theoretical framework and
analysis of formal inductive synthesis. 
Most instances of
inductive synthesis in the literature rely on an oracle that answers
different types of queries. In order to capture these various synthesis
methods in a unifying framework, we formalize the notion of 
{\it {oracle-guided inductive synthesis}} ($\ogis$). 
While we defer a detailed treatment of $\ogis$ to
Section~\ref{sec:ogis}, we point out three dimensions in
which $\ogis$ techniques differ from each other:
\begin{mylist}
\item {\em Characteristics of concept class:} The concept class for
synthesis may have different characteristics depending on the
application domain. For instance, the class of programs from which the synthesizer
must generate the correct one may be finite, as in the synthesis of
bitvector programs~\cite{asplos06,jha-icse10,jha-11}, or infinite, as
in the synthesis of guards for hybrid
automata~\cite{jha-cps10,jha-emsoft11}. In the former
case, termination is easily guaranteed, but it is not obvious for
the case of infinite-size concept classes.

\item {\em Query types:} Different applications may impose differing
constraints on the capabilities of the oracle. In some cases, the
oracle may provide only positive examples. When verification engines
are used as oracles, as is typical in formal synthesis, the oracle may
provide both positive examples and counterexamples which refute candidate 
programs. More fine-grained properties
of queries are also possible --- for instance, an oracle may permit
queries that request not just any counterexample, but one that
is ``minimum'' according to some cost function.

\item {\em Resources available to the learning engine:} As noted
above, the learning algorithms in formal inductive synthesis tend to
be general-purpose decision procedures. Even so, for tractability,
certain constraints may be placed on the resources available to the
decision procedure, such as time or memory available. For example, one
may limit the decision procedure to use a finite amount of memory,
such as imposing an upper bound on the number of (learned) clauses for a
SAT solver.

\end{mylist}
We conduct a theoretical study of $\ogis$ by examining the impact of
variations along the above three dimensions. Our work has a particular
focus on {\em counterexample-guided inductive synthesis}
(CEGIS)~\cite{asplos06}, a particularly popular and effective
instantiation of the $\ogis$ framework.
When the concept class is infinite size, termination of CEGIS is
not guaranteed. We study the relative strength of different versions
of CEGIS, with regards to their termination guarantees. The versions
vary based on the type of counterexamples one can obtain from the
oracle. We also analyze the impact of finite versus infinite memory
available to the learning algorithm to store examples and hypothesized programs/concepts.
Finally, when the concept class is finite size, even though
termination of CEGIS is guaranteed, the speed of termination can
still be an issue. In this case, we draw a connection between the
number of counterexamples needed by a CEGIS procedure and the notion
of {\em teaching dimension}~\cite{goldman-td} previously introduced in
the machine learning literature.
 
To summarize, we make the following specific contributions in this paper:
\begin{mylist}
\item We define the {\em formal inductive synthesis} problem and
propose a class of solution techniques termed as {\it {Oracle-Guided Inductive Synthesis} } ($\ogis$).
We illustrate
how $\ogis$ generalizes instances of concept learning in machine learning/artificial intelligence as well as 
synthesis techniques developed using formal methods. We provide examples of synthesis techniques from literature
and show how they can be represented as instantiations of \ogis.

\item We perform a theoretical comparison of different instantiations
of the $\ogis$ paradigm in terms of their {\em synthesis power}. 
The synthesis power of an $\ogis$ technique is defined as the class of
concepts/programs (from an infinite concept class) that can be
synthesized using that technique. We
establish the following specific novel theoretical results:
\begin{myitemize}
\item For learning engines that can use
unbounded memory, the power of synthesis engines using oracle  
that provides arbitrary counterexamples or minimal counterexamples is
the same. But this is strictly more  
powerful than using oracle which provides counterexamples which are bounded by
the size of the positive examples.

\item For learning engines that use
bounded memory, the power of synthesis engines using arbitrary
counterexamples 
or minimal counterexamples is still the same. 
The power of synthesis engines using counterexamples bounded by
positive examples is not comparable 
to those using arbitrary/minimal counterexamples. Contrary to
intuition, using counterexamples bounded by positive 
examples allows one to synthesize programs from program classes which
cannot be synthesized using arbitrary or minimal counterexamples.

\end{myitemize}

\item For finite concept classes, we prove the NP hardness of the problem
of solving the formal inductive synthesis problem for finite domains
for a large class of $\ogis$ techniques.
We also show that the teaching
dimension~\cite{goldman-td} of the concept class is a lower bound on
the number of counterexamples needed for a CEGIS technique to
terminate (on an arbitrary program from that class).

\end{mylist}


The rest of the paper is organized as follows.
We first present the 
{\it {Oracle Guided Inductive Synthesis}} ($\ogis$) paradigm in
Section~\ref{sec-ogis}. We discuss related work in Section~\ref{sec-rel}. We present the notation
and definitions used for theoretical analysis in Section~\ref{sec-not} followed by the theoretical
results and their proofs in Section~\ref{sec-res} and Section~\ref{sec-finres}. 
We summarize our results and discuss open problems in
Section~\ref{sec-con}. A preliminary version of this paper appeared in
the SYNT 2014 workshop~\cite{jha-synt14}.

\section{Oracle-Guided Inductive Synthesis: $\ogis$} 
\label{sec-ogis}
\label{sec:ogis}

We begin by defining some basic terms and notation.
Following standard terminology in
the machine learning theory community~\cite{angluin88}, we define 
a concept $\concept$ as a set of examples drawn from a domain of examples
$\domain$. In other words, $\concept \subseteq \domain$. An example
$\ex \in \domain$ can be viewed as an input-output behavior
of a program; for example, a (pre, post) state for a terminating
program, or an input-output trace for a reactive program.
\revision{Thus, in this paper, we ignore syntactic issues in representing concepts
and model them in terms of their semantics, as a set of behaviors.} 
The set of all possible concepts is termed the {\em concept class},
denoted by $\class$. Thus, $\class \subseteq \powset{\domain}$.
\revision{The concept class may either be specified in the original
synthesis problem or arise as a result of a structure hypothesis that
restricts the space of candidate concepts.}
Depending on the application domain, $\domain$ can be finite or
infinite. The concept class $\class$ can also be finite or
infinite. Note that it is possible to place (syntactic) restrictions 
on concepts so that $\class$ is finite even when $\domain$ is infinite. 

One key distinguishing characteristic between traditional machine
learning and formal inductive synthesis is the presence of 
an explicit formal
specification in the latter. We define a specification $\Spec$ 
as a set of ``correct'' concepts, i.e., $\Spec \subseteq
\class \subseteq \powset{\domain}$. 
Any example $\ex \in \domain$ such that there is a concept $c \in
\Spec$ where $\ex \in c$ is called a {\em positive example}. 
Likewise, an example $\ex$ that is not contained
in any $c \in \Spec$ is a {\em negative example}. 
We will write $\ex \satEx \Spec$ to denote that $\ex$ is a positive example.
\acta{
An example that is specified to be either positive or negative is 
termed a {\em labeled example}.
}

Note that standard practice in formal methods is to define a specification
as a set of examples, i.e., $\Spec \subseteq \domain$. This is
consistent with most properties that are {\em trace
properties}, where $\Spec$ represents the set of allowed behaviors ---
traces, (pre,post) states, etc. --- of the program. 
However, certain practical properties of systems, e.g.,
certain security policies, are not trace properties (see, e.g.,~\cite{clarkson-jcs10}), and
therefore we use the more general definition of a specification.

We now define what it means for a concept to satisfy $\Spec$.
Given a concept $\target \in \class$ we say that $\target$ satisfies
$\Spec$ iff $\target \in \Spec$.
If we have a complete specification, it means that $\Spec$ is a
singleton set comprising only a single allowed concept.
In general, $\Spec$ is likely to be a partial specification
that allows for multiple correct concepts.


We now present a first definition of 
the {\em formal inductive synthesis} problem:
\begin{quote}
Given a {\em concept class} $\class$ and a domain of examples
$\domain$, 
the formal inductive synthesis problem is to find, 
using only a subset of examples from $\domain$,
a {\em target concept} $\target \in \class$ that
satisfies a specification $\Spec \subseteq \class$.
\end{quote}
\revision{This definition is reasonable in cases where only elements of
$\domain$ can be accessed by the synthesis engine --- the common case
in the use of machine learning methods.
However, existing formal verification and synthesis methods can use a somewhat richer set of inputs,
including Boolean answers to equivalence (verification) queries with respect to
the specification $\Spec$, as well as verification queries with respect to other
constructed specifications. Moreover, the synthesis engine typically does not 
directly access or manipulate the specification $\Spec$. 
In order to formalize this richer source of inputs as well as the indirect access
to $\Spec$, we introduce the concept
of an {\em oracle interface}. 
\begin{definition}
An {\em oracle interface} $\orint$ is a subset of $\defquery \times \defresponse$ where
$\defquery$ is a set of query types, 
$\defresponse$ is a corresponding set of response types, and
$\orint$ defines which pairs of query and response types are semantically well-formed.
\qed
\label{def:oracle-interface}
\end{definition}
A simple instance of an oracle interface is one with a single query \acta{type} that
returns positive examples from $\domain$. In this case, the synthesis problem is to learn
a correct program from purely positive examples. 
The more common case in machine learning (of classifiers) is to have an oracle that supports two
kinds of queries, one that returns positive examples and another that returns negative examples.
As we will see in Sec.~\ref{sec:ogis-features}, there are richer types of queries that are
commonly used in formal synthesis. For now, we will leave $\defquery$
and $\defresponse$ as abstract sets.

\acta{
Implementations of the oracle interface can be nondeterministic algorithms
which exhibit nondeterministic choice in the stream of queries and responses.
We define
the notion of {\it nondeterministic mapping} to represent such algorithms.

\begin{definition}
A nondeterministic mapping $F: I \rightarrow O $ takes as input 
$i \in I$ and produces an output $o \in O(i) \subseteq O$ where 
$O(i)$ is the set of all valid outputs corresponding to input $i$
in $F$. 
\end{definition}
}

With this notion of an oracle interface, we now introduce our definition of formal
inductive synthesis (\FIS):
\acta{
\begin{definition}
Consider a {\em concept class} $\class$, a domain of examples
$\domain$, a specification $\Spec$, and an oracle interface $\orint$. 
The formal inductive synthesis problem is to find
a {\em target concept} $\target \in \class$ that
satisfies $\Spec$,
given only $\orint$ and $\class$. 
In other words, 
 $\domain$ and $\Spec$ can be accessed  only through $\orint$.
\qed
\label{def:formal-inductive-synthesis}
\end{definition}
Thus, an instance of \FIS is defined in terms of the tuple 
$\langle \class, \domain, \Spec, \orint \rangle$.
We next introduce a family of solution techniques for the \FIS problem.
}
A FIS problem instance defines an oracle interface 
and a solution technique for that problem instance 
can access the domain $\domain$ 
and the specification $\Spec$ only through that interface. 
}

\subsection{\ogis: A family of synthesizers}
\label{sec:ogis-features}

Oracle-guided inductive synthesis (\ogis) is an approach to solve the
formal inductive synthesis problem defined above, encompassing a
family of synthesis algorithms. 

\comment{
\begin{figure}[htb!]
\centering%
\scalebox{0.7}{\includegraphics{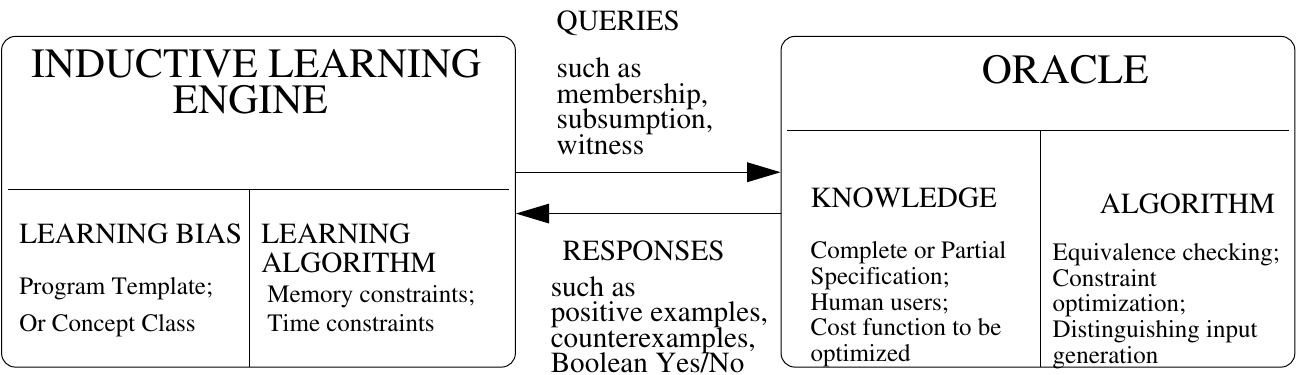}}
\caption{Oracle Guided Inductive Synthesis}
\label{fig:ogis}
\end{figure}
} 
\begin{figure}[htb!]
\centering
\includegraphics[height=5cm]{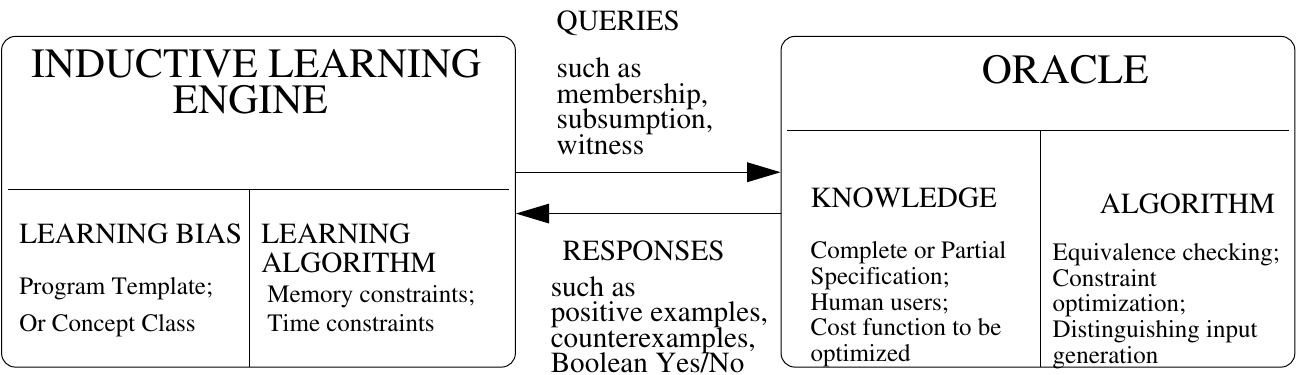}
\caption{Oracle Guided Inductive Synthesis}
\label{fig:ogis}
\end{figure}

As illustrated in Figure~\ref{fig:ogis}, \ogis comprises
two key components: an {\em inductive learning engine} (also sometimes
referred to as a ``Learner'')
and an {\em oracle} (also referred to as a ``Teacher''). 
The interaction between the learner and the oracle is in the form of
a {\em dialogue} comprising 
{\em queries} and {\em responses}. The oracle is defined by the types of
queries that it can answer, and the properties of its responses.
Synthesis is thus an iterative process: 
at each step, 
the learner formulates and sends a query to the
oracle, and the oracle sends its response.
For formal synthesis, the oracle is also tasked with determining
whether the learner has found a correct target concept. Thus, the
oracle implicitly or explicitly maintains the specification $\Spec$
and can report to the learner when it has terminated with a correct
concept. 

\revision{
We first formalize the notions of learner and oracle.
\acta{Let $\queries$ be a set of queries of types $\defquery$,
and $\responses$ be a set of responses of types
$\defresponse$.} 
We allow both $\queries$
and $\responses$ to include a special element $\bot$ indicating the
absence of a query or response.
An element $(q,r) \in \queries \times \responses$ 
is said to {\em conform} to an oracle interface $\orint$ if
$q$ is of type $q_t$, $r$ is of type $r_t$ and $(q_t,r_t) \in \orint$.}
\revise{
A {\em valid dialogue pair} for an oracle interface $\orint$, 
denoted $d$, is a query-response
pair  $(q,r)$ 
such that $q \in \queries$, $r \in \responses$ and
 $(q,r) $ conforms to $\orint$.
 The set of valid dialogue pairs for an oracle interface
 is denoted by $\dialogue$ and }
\revise{$\dialogue^*$ denotes the set of {\em valid dialogue sequences}
--- finite sequences of valid dialogue pairs.
If $\dseq \in \dialogue^*$ is a valid dialogue sequence,
$\dseq[i]$ denotes a sub-sequence of $\dseq$ of length $i$
and $\dseq(i)$ denotes the $i$-th dialogue pair in the sequence.}

\begin{definition}
An {\em oracle} is a \acta{ nondeterministic } mapping $\deforacle: \dialogue^* \times \queries
\rightarrow \responses$.
\acta{ 
$\deforacle$ is {\it consistent} with a given interface $\orint$
iff given a valid dialogue sequence $\dseq$ and a query $q$ of type
$q_t$, $\deforacle(\dseq,q)$ is a response of type $r_t$ where
$(q_t, r_t) \in \orint$.
%
A {\em learner} is a 
nondeterministic mapping $\deflearner: \dialogue^* \rightarrow \queries \times \class$. 
$\deflearner$ is {\it consistent} with a given interface $\orint$
iff given a valid dialogue sequence $\dseq$, 
$\deflearner(\dseq) = (q,c)$ where $q \in \queries$ has type $q_t$ such
that there exists a response type $r_t$ s.t. $(q_t, r_t) \in \orint$.
%
%
}
\qed
\label{def:ogisoraclelearn}
\end{definition}
We will further assume in this paper that the oracle $\deforacle$ is {\em sound}, meaning
that it gives a correct response to every query it receives. For example, if asked for
a positive example, $\deforacle$ will not return a negative example instead. 
This notion is left informal
for now, since a formalization requires discussion of specific queries
and is orthogonal to the results in our paper.

Given the above definitions, we can now define the \ogis approach formally.
\acta{
\begin{definition}
Given a \FIS $\langle \class, \domain, \Spec, \orint \rangle$, 
an {\em oracle-guided inductive synthesis} (\ogis) procedure (engine)
is a tuple $\langle \deforacle, \deflearner \rangle$,
comprising 
an {\em oracle} $\deforacle: \dialogue^* \times \queries \rightarrow \responses$ and 
a {\em learner} $\deflearner: \dialogue^* \rightarrow \queries \times \class$,
where the oracle and learner are consistent with the given oracle
interface $\orint$ as defined above.
\qed
\label{def:ogis}
\end{definition}
In other words, an \ogis engine comprises an oracle $\deforacle$ that
maps a ``dialogue history'' and a current query to a response, and
a learner $\deflearner$ that, given a dialogue history, outputs a
hypothesized concept along with a new query. 
Upon convergence, the final concept output by $\deflearner$ is the
output of the \ogis procedure.

We also formalize the definition of when an \ogis engine solves an \FIS problem.
\begin{definition}
 A dialogue sequence $\dseq \in \dialogue^*$  
 corresponding to \ogis procedure $\langle \deforacle, \deflearner \rangle$ is such that
 $\dseq(i)$ is $(q,r)$ where $\deflearner(\dseq[i-1]) = (q,c)$
 for some query $q \in \queries$ and some concept $c \in \class$,
 and $\deforacle(\dseq[i-1], q) = r $. 

 The \ogis procedure  $\langle \deforacle, \deflearner \rangle$ is said
 to {\em solve} the \FIS problem {\em with dialogue sequence 
 $\dseq$} if there exists an $i$ such that 
 $\deflearner(\dseq[i]) = (q,c)$, $c \in \class$ and $c$ satisfies
 $\Spec$, and for all $j > i$, 
 $\deflearner(\dseq[j]) = (q',c)$, that is, the \ogis procedure 
 converges to a concept $\concept \in \class$ that satisfies $\Spec$.

 The \ogis procedure  $\langle \deforacle, \deflearner \rangle$ is said
 to {\em solve} the \FIS problem if there exists a dialogue sequence $\dseq$
 with which it solves that problem.

\qed
\label{def:ogis-solve}
\end{definition}
}

The convergence and computational complexity of an \ogis procedure is
determined by the nature of the \FIS problem along with three 
factors: (i) the complexity of each
invocation of the learner $\deflearner$; (ii) the complexity of each invocation of
the oracle $\deforacle$, and (iii) the number of iterations (queries,
examples) of the loop before convergence. 
We term first two factors as {\em learner
complexity} and {\em oracle complexity}, and the third as {\em sample
complexity}. Sometimes, in \ogis procedures, oracle complexity is
ignored, so that we simply count calls to the oracle rather than the
time spent in each call.

An \ogis procedure is defined by properties of the learner and the
oracle. Relevant properties of the learner include (i) its {\em inductive bias} 
that restricts its search to a particular
family of concepts and a search strategy over this space, and
(ii) {\em resource constraints}, such as finite or
infinite memory. 
Relevant properties of the oracle include the types of queries it
supports and of the responses it generates.
We list below the common query and response types. In each case, the query
type is given in square brackets 
as a template comprising a query name along with the types of
the formal arguments to that query, e.g., examples $\ex$ or concepts $\concept$.
An instance of each of these types, that is, a query, 
is formed by substituting a specific arguments (examples, concepts, etc.) for 
the formal arguments.
\begin{mylist}
\item {\em Membership query:} [$\qmem(\ex)$] 
The learner selects an example $\ex$ and
asks ``Is $\ex$ positive or negative?'' The oracle responds with a
label for $\ex$, indicating whether $\ex$ is a positive or negative example.

\item {\em Positive witness query:} [$\qposwit$] 
The learner asks the oracle ``Give me a
positive example''. The oracle responds with an example $\ex \satEx
\Spec$, if one exists, and with $\bot$ otherwise.

\item {\em Negative witness query:} [$\qnegwit$] 
The learner asks the oracle ``Give me a
negative example''. The oracle responds with an example $\ex \not \satEx
\Spec$, if one exists, and with $\bot$ otherwise.

\acta{
\item {\em Counterexample query:} [$\qceq(\concept)$] 
The learner proposes a candidate
concept $\concept$ and asks ``Does the oracle have a counterexample 
demonstrating that $\concept$ is incorrect?'' (i.e., ``proof that $ \concept \not \in \Spec$?''). 
If the oracle can find a counterexample $\ex$ to 
$\concept \not \in \Spec$, 
the oracle provides the counterexample.
Otherwise, if the oracle cannot find any counterexample, it
responds with $\bot$.
Such a query allows us to accurately model the working of counterexample-guided
synthesis techniques such as~\cite{jin-hscc13} where the verification problem
is undecidable but, if a counterexample is reported, it is a true counterexample.
} 

\item {\em Correctness query:} [$\qcorr(\concept)$] 
The learner proposes a candidate
concept $\concept$ and asks ``Is $\concept$ correct?'' (i.e., ``does
it satisfy $\Spec$?''). 
If so, the oracle responds ``YES'' (and the synthesis can terminate). 
If it is not so,
the oracle responds ``NO'' and provides the counterexample.
Here $\ex$ is an example such that either $\ex \in \concept$ but $\ex \not \satEx
\Spec$, or $\ex \not \in \concept$ and there exists some other concept
$\concept' \in \Spec$ containing $\ex$. 
\acta{This query is a stronger query than
counterexample query as it is guaranteed to provide a counterexample
whenever the proposed $\concept$ is not correct.}
 
For the special case of trace properties, the correctness query can
take on specific forms. One form is termed the {\em equivalence
query}, denoted $\qeq$, where the counterexample is in the symmetric 
difference of the single correct target concept and $\concept$.
The other is termed the {\em subsumption query}, denoted $\qsub$,
where the counterexample is a negative example present in $\concept$,
and is used when $\Spec$ is a partial specification admitting several correct concepts. It is important to note that, in the general case, a verification query does not, by itself,
specify any label for a counterexample. One may need an additional
membership query to generate a label for a counterexample.

\item {\em Crafted Correctness (Verification) query:} [$\qccorr(\hat{\concept},\hat{\Spec})$]
As noted earlier, oracles used in formal inductive synthesis tend to be general-purpose decision
procedures. Thus, they can usually answer not only verification queries with respect to the
specification $\Spec$ for the overall FIS problem, but also verification queries
for specifications crafted by the learner. We refer to this class of queries as 
{\em crafted correctness/verification queries}. The learner asks ``Does $\hat{\concept}$ satisfy
$\hat{\Spec}$?'' for a crafted specification $\hat{\Spec}$ and a crafted concept $\hat{\concept}$.

As for $\qcorr$ one can define as special cases a crafted equivalence query type $\qceq$
and a crafted subsumption query type $\qcsub$.

\comment{
\item {\em Equivalence query:} The learner proposes a candidate
concept $\concept$ and asks ``Is $\concept = \Spec$?'' If so,
the oracle responds ``YES'' (and the synthesis can terminate). If not,
the oracle responds ``NO'' and provides a counterexample, an example
$\ex \in \concept \ominus \Spec$ where $\ominus$ denotes symmetric
difference. This query is relevant in the case when the specification
$\Spec$ is complete.

\item {\em Subsumption query:} This query is similar to the case of
the equivalence query but in the case when $\Spec$ is
partial. Formally, the learner proposes a candidate concept $\concept$
and asks ``Is $\concept \subseteq \Spec$?'' If so, the oracle responds
``YES'' (and synthesis can terminate). Otherwise, the oracle responds ``NO'' and provides a
counterexample, an element $\ex \in \concept \setminus \Spec$.
}

\item {\em Distinguishing input query:} [$\qdiff(\ExSet,\concept)$]
In this query, the learner supplies a finite set of examples $\ExSet$ and a
concept $\concept$, where $\ExSet \subseteq \concept$, and asks ``Does
there exist another concept $\concept'$ s.t. $\concept \not =
\concept'$ and $\ExSet \subseteq \concept'$?'' If so, the oracle
responds ``YES'' and provides both $\concept'$ and an example $\ex \in
\concept \ominus \concept'$. The example $\ex$ forms a so-called
``distinguishing input'' that differentiates the two concepts
$\concept$ and $\concept'$. 
If no such $\concept'$ exists, the oracle responds ``NO''.

\comment{
\item {\em Distinguishing input query:} ($\qdiff$)
In this query, the learner supplies a finite set of examples $\ExSet$ and a
concept $\concept$, where $\ExSet \subseteq \concept$, and asks ``Does
there exist another concept $\concept'$ s.t. $\concept \not =
\concept'$ and $\ExSet \subseteq \concept'$?'' If so, the oracle
responds ``YES'' and provides both $\concept'$ and an example $\ex \in
\concept \ominus \concept'$. The example $\ex$ forms a so-called
``distinguishing input'' that differentiates the two concepts
$\concept$ and $\concept'$. 
If no such $\concept'$ exists, the oracle responds ``NO''.
} 

The distinguishing input query has been found useful in scenarios
where it is computationally hard to check correctness using the
specification $\Spec$, such as in malware deobfuscation~\cite{jha-icse10}.

\end{mylist}
The query/response types $\qmem$, $\qposwit$, $\qnegwit$, $\qceq$, $\qcorr$, $\qccorr$ and
$\qdiff$ listed above are not meant to be exhaustive. Any subset of such
types can form an oracle interface $\orint$. 
We note here that, in the machine learning theory community, there
have been thorough studies of query-based learning; see Angluin's
review paper~\cite{Angluin-queries} for details. However, in our formalization
of \ogis, new query types such as $\qccorr$ and $\qdiff$ are possible due
to the previously-identified key differences with traditional machine learning
including the general-purpose nature of oracle implementations and
the ability to select or even design the oracle.
%
Moreover, as we will see, our theoretical analysis 
raises the following questions that are pertinent in the
setting of formal synthesis where the learner and oracle are typically
implemented as general-purpose decision procedures:
\begin{myitemize}
\item {\em Oracle design:} 
When multiple valid responses can be made to a query, which ones are
better, in terms of convergence to a correct concept (convergence and complexity)?

\item {\em Learner design:} 
How do resource constraints on the learner or its choice of search
strategy affect convergence to a correct concept?

\end{myitemize}

\subsection{Examples of \ogis}
\label{sec:ogis-examples}

We now take three example synthesis techniques previously presented in literature and illustrate
how they instantiate the $\ogis$ paradigm. These techniques mainly differ in the oracle interface
that they employ.

\begin{example} Query-based learning of automata~\cite{angluin1987learning}:\\
{\em Angluin's classic work on learning deterministic finite automata
(DFAs) from membership and equivalence
queries~\cite{angluin1987learning} is an instance of \ogis with $\orint = \{\qmem, \qeq\}$.
The learner is a custom-designed algorithm called $L^*$, whereas the oracle is
treated as a black box that answers the membership and equivalence
queries; in particular, no assumptions are made about the form of counterexamples.
Several variants of $L^*$ have found use in the formal verification
literature; see~\cite{giannakopoulou-fmsd08} for more information.}
\end{example}

\begin{example} Counterexample-guided inductive synthesis (CEGIS)~\cite{asplos06}:\\
\em{
\acta{
 CEGIS was originally proposed as an algorithmic method for program
synthesis where the specification is given as a reference program and
the concept class is defined using a partial program, also referred to
as a ``sketch''~\cite{asplos06}. 
It has since proved very versatile,
also applying to partial specifications (see, e.g.,~\cite{jin-hscc13}) 
and other ways of providing syntax guidance; see~\cite{alur-fmcad13} for a more detailed treatment.
In CEGIS, the learner (synthesizer) interacts with a ``verifier' that can take in a candidate program and
a specification, and 
try to find a counterexample showing that the candidate program
does not satisfy the specification. 
In CEGIS, the learner is typically implemented on top of a
general-purpose decision procedure such as a SAT solver, SMT solver,
or model checker. The oracle (verifier) is also implemented similarly. 
In addition to a counterexample-generating oracle, many instances of
CEGIS also randomly sample positive examples (see Sec. 5.4 of~\cite{asplos06}
and Fig. 3 of~\cite{jin-hscc13}). Moreover, the counterexample-generating
oracle is not required to be a sound verifier that can declare correctness
(e.g., see~\cite{jin-hscc13}).
Thus, we model CEGIS as an instance of \ogis 
with $\orint = \{\qposwit, \qceq\}$.
}

\acta{
As noted earlier, if the verifier is sound (can prove correctness of candidate
concept), then $\qceq$ can be substituted by $\qcorr$.} Moreover,
general-purpose verifiers 
typically support not only correctness queries with respect to the original specification,
but also crafted correctness queries, as well as membership queries,
which are special cases of the verification problem where the specification 
is checked on a single input/output behavior. We term an instantiation of CEGIS with
these additional query types as {\em generalized CEGIS},
which has an oracle interface
$\orint = \{\qposwit, \qcorr, \qccorr, \qmem\}$. We will restrict our
attention in this paper to the standard CEGIS.}

\end{example} 

\begin{example} Oracle-guided program synthesis using distinguishing inputs~\cite{jha-icse10}:\\
{\em Our third example is an approach to program synthesis that uses distinguishing inputs when 
a complete specification
is either unavailable or it is expensive to verify a candidate program
against its specification~\cite{jha-icse10}. 
In this case,
distinguishing input queries, combined with witness and membership queries, provide a way to quickly generate a
corpus of examples that rule out incorrect programs. When there is only a single
program consistent with these examples, only then does a correctness 
query need to be made to ascertain its correctness. Thus, the oracle interface 
$\orint = \{ \qposwit, \qdiff, \qmem, \qcorr \}$ with $\qcorr$ being used sparingly. 
The learner and the oracle are implemented using SMT solving.}
\end{example}

\comment{ \begin{example} ICE learning for invariant synthesis~\cite{garg-cav14}:\\
{\em Recently, an approach to invariant generation has been proposed that
uses learning from ``implications, counterexamples, and examples'' (ICE)
--- positive and negative examples (states)
coupled with queries to a solver to determine
whether a hypothesized invariant is inductive. When the latter query is
answered in the negative, it is accompanied by a counterexample that is in
the form of a pair of states (seemingly different from the positive/negative
examples that are single states), and which also does not indicate any specific
+/- label for these states. However, ICE learning is also an instance of
\ogis, when one observes that correctness queries in general do not provide labels.
Thus, $\orint = \{ \qposwit, \qnegwit, \qcorr \}$. Additionally, one can lift
the domain of examples from single states to pairs of states, and define an
corresponding concept class equivalent to the class of candidate invariants.}
\end{example} }


\subsection{Counterexample-Guided Inductive Synthesis (CEGIS)}

Consider the CEGIS instantiation of the \ogis framework.
In this paper, we consider a general setting where the concept class
$\class$ is the set of programs corresponding to the set of {\em
recursive (decidable) languages}; thus, it is infinite. 
The domain $\domain$ of examples is also infinite. 
We choose such an expressive concept class and domain 
because we want to compare how the power of CEGIS varies as we vary
the oracle and learner. More specifically, we vary the {\em nature of
responses} from the oracle to 
correctness 
and witness queries, and the {\em memory available} to the learner.

For the oracle, we consider four different types of counterexamples 
that the oracle can provide in response to a 
correctness  
query. Recall that in formal synthesis, oracles are general-purpose
verifiers or decision procedures whose internal heuristics may
determine the type of counterexample obtained.
Each type describes a different oracle and hence, a different
flavor of CEGIS. Our goal is to compare these synthesis techniques and establish whether
one type of counterexample allows the synthesizer to successfully learn more programs than 
the other. The four kinds of counterexamples considered in this paper are as follows:
\begin{mylist}
\item {\it Arbitrary counterexamples:} This is the ``standard'' CEGIS
technique (denoted $\tcegis$) that makes no assumptions on the form of
the counterexample obtained from the oracle. Note however that our
focus is on an infinite concept class, whereas most practical
instantiations of CEGIS have focused on finite concept classes; thus,
convergence is no longer guaranteed in our setting. 
This version of CEGIS serves as the baseline for comparison against other synthesis techniques.

\item {\it Minimal counterexamples:} 
\acta{
We require that the verification oracle provide a counterexample from
$\domain$ which
is minimal for a given ordering over $\domain$. The size of examples can be used
for ordering. The exact definition of ``size'' is left abstract and
can be defined suitably in different contexts. The intuition
is to use counterexamples of smaller size which eliminates more
candidate concepts. }
Significant effort has been made on improving validation
engines to produce counterexamples which aid debugging by localizing
the error~\cite{antonio13,ChenSMV10}. The use of counterexamples in
$\tcegis$ conceptually is an iterative repair process and hence, it is
natural to extend successful error localization and debugging
techniques to inductive synthesis.  

\item {\it Constant-bounded counterexamples:} Here the ``size'' of the counterexamples produced by the
verification oracle is bounded by a constant. This is motivated by the use of bounds in formal
verification such as bounded model checking~\cite{biere-bmc} and 
bug-finding in concurrent programs~\cite{qadeer-bounded} using bounds
on context switches.

\item {\it Positive-bounded counterexamples:} Here the counterexample produced by the validation 
engine must be smaller than a previously seen positive example. This is motivated from the 
industrial practice of validation by simulation where the system is often simulated to a finite length to 
discover bugs. The length of simulation often depends on the traces which illustrate known positive behaviors.
It is expected that errors will show up if the system is simulated up to the length of the largest positive trace.
Mutation-based software testing and symbolic execution also has a
similar flavor, where a sample correct execution is mutated to find bugs.

\end{mylist}

In addition to the above variations to the oracle, we also consider two kinds of learners that
differ based on their ability to store examples and counterexamples:
\begin{mylist}
\item {\it Infinite memory:} In the typical setting of CEGIS, the
learner is not assumed to have any memory bound, allowing the learner
to store as many examples and counterexamples as needed.
Note that, for an infinite domain, this set of examples can grow unbounded.

\item {\it Finite memory:} A more practical setting is one where the
learner only has a finite amount of memory, and therefore can only
store a finite representation of examples or hypothesized programs. 
This notion of finite memory is similar to that used classically for language
learning from examples~\cite{wiehagen76}.
We give the first theoretical
results on the power of CEGIS and its variants, for general program synthesis, in this
restricted setting.

\end{mylist}

We introduce notation to refer to these variants in a more compact manner.
The synthesis engine using arbitrary counterexamples and with infinite memory 
is denoted as $\engine_\infcegis$. 
The variant of the synthesis engine which is restricted to
use finite memory is referred to as $\engine_\cegis$. 
Similarly, the synthesis engine
using minimal counterexamples and infinite memory is called minimal counterexample guided inductive
synthesis ($\engine_\infmncegis$). The variant of this engine using finite memory is referred to as $\engine_\mncegis$. 
\revise{
The synthesis engine using counterexamples which are smaller than a
fixed constant
is called a constant bounded counterexample guided inductive synthesis, 
and is denoted as $\engine_\infbcegis$ if the memory is not
finite and $\engine_\bcegis$ if the memory is finite. 
}
The synthesis engine using counterexamples which are smaller than the largest positive examples 
is called positive-history bounded counterexample guided inductive 
synthesis, and is denoted as $\engine_\infhmncegis$ if the memory is not
finite and $\engine_\hmncegis$ if the memory is finite. 

For the class of programs corresponding to the set of recursive languages,
our focus is on {\em learning in the limit}, that is, whether the synthesis technique converges to
the correct program or not (see Definition~\ref{defn:convergence} in Sec.~\ref{sec:not} for a formal
definition). This question is non-trivial since our concept class is not finite. 
In this paper,
we do not discuss computational complexity of synthesis, and the impact of different types
of counterexamples on the speed of convergence. Investigating the computational complexity for
concept classes for which synthesis is guaranteed to terminate is left as a topic
for future research.
 
We also present an initial complexity analysis for $\ogis$ in case of finite concept classes.
The decidability question for finite class of programs is trivial since convergence 
is guaranteed as long as the queries provide new examples or some new information about the 
target program. But the speed at which the synthesis approach converges remains relevant even 
for finite class of programs. We show
that the complexity of these techniques is related to well-studied
notions in learning theory such as the
{\it Vapnik-Chervonenkis dimension}~\cite{Blumer-vp} and the {\it teaching dimension}~\cite{goldman-td}.

\section{Background and Related Work} \label{sec-rel}

In this section, we contrast the contributions of this paper with the
most closely related work
and also provide some relevant background.

\subsection{Formal Synthesis}

The past decade has seen an explosion of work in program
synthesis
(e.g.~\cite{sketching:pldi05,asplos06,jha-icse10,SGF10,KMPS12,udupa-pldi13}.
Moreover, there has been a realization that many of the 
trickiest steps in formal verification involve synthesis of artifacts
such as inductive invariants, ranking functions, assumptions,
etc.~\cite{seshia-dac12,grebenshchikov2012synthesizing}. 
Most of these efforts have focused on solution techniques for
specific synthesis problems. There are two main unifying characteristics across
most of these efforts: (i) syntactic restrictions on
the space of programs/artifacts to be synthesized in the form of
templates, sketches, component libraries, etc., and (ii) the use of
inductive synthesis from examples. 
The recent work on {\em syntax-guided synthesis}
(SyGuS)~\cite{alur-fmcad13} is an attempt to capture these disparate
efforts in a common theoretical formalism. While SyGuS is about
formalizing the synthesis {\it problem}, the present paper focuses on
formalizing common ideas in the {\it solution techniques}. Specifically,
we present \ogis as a unifying formalism for different solution
techniques, along with a theoretical analysis of different
variants of CEGIS, the most common instantiation of \ogis.
In this sense, it is complementary to the SyGuS effort.

\subsection{Machine Learning Theory}

Another related area is the field of machine learning, particularly
the theoretical literature. In Section~\ref{sec:intro}, we outlined
some of the key differences between the fields of formal inductive
synthesis and that of machine learning. Here we focus on the sub-field
of {\em query-based learning} that is the closest to the \ogis
framework. The reader is referred to Angluin's excellent papers on the
topic for more background~\cite{angluin88,Angluin-queries}. 

A major difference between the query-based learning literature and
our work is in the treatment of oracles, specifically, how much control
one has over the oracle that answers queries. In query-based learning, the
oracles are treated as black boxes that answer particular types of
queries and only need to provide one valid response to a query. 
Moreover, it is typical in the query-based learning literature 
for the oracle to be specified a priori as part of the problem formulation. 
In contrast, in our \ogis framework, designing a synthesis
procedure involves also designing or selecting an oracle.
The second major difference is that the query-based learning
literature focuses on specific concept classes and proves convergence
and complexity results for those classes. In contrast, our work proves
results that are generally applicable to programs corresponding to
recursive languages.

\subsection{Learning of Formal Languages}

The problem of learning a formal language from examples is a classic
one.
We cover here some relevant background material.

Gold~\cite{gold67limit} considered the problem of learning formal languages from
examples. Similar techniques have been studied elsewhere in 
literature~\cite{JantkeB81,Wiehagen90,blum75,angluin1980inductive}.
The examples
are provided to learner as an infinite stream.
The learner is assumed to have unbounded memory and can store all the examples.
This model is unrealistic in a practical setting but provides useful theoretical understanding
of inductive learning of formal languages. 
Gold defined a class of languages to be {\it identifiable in the limit} if there is a learning procedure which identifies the grammar
of the target language from the class of languages using a stream of input strings. The languages learnt using only positive
examples were called {\it text learnable} and the languages which require both positive and negative examples were termed
{\it informant learnable}. 
None of the standard classes of formal languages are identifiable in the limit from
text, that is, from only positive examples~\cite{gold67limit}. This
includes regular languages, context-free languages and
context-sensitive languages. 

A detailed survey of classical results in learning from positive examples
is presented by Lange et al.~\cite{Lange08}. The results summarize learning power with different limitations such as
the inputs having certain noise, that is, a string not in the target
language might be provided as a positive example with a small
probability. Learning using positive as well as negative examples has
also been well-studied in literature. A detailed survey is presented
in \cite{jain1999systems} 
and \cite{lange2000algorithmic}.  
Lange and Zilles~\cite{Lange04formallanguage} relate Angluin-style query-based learning with
Gold-style learning. They establish that any query learner using superset queries can be
simulated by a Gold-style learner receiving only positive data. But there are concepts
learnable using subset queries but not Gold-style learnable from positive data only.
Learning with equivalence queries coincides with Gold's model of limit learning from
positive and negative examples, while learning with membership queries equals 
finite learning from positive data and negative data.
In contrast to this line of work, we present a general framework $\ogis$ 
to learn programs or languages and Angluin-style or Gold-style 
approaches can be instantiated in this framework.
Our theoretical analysis focusses on varying the oracle and the nature of counterexample
produced by it to examine the impact of using different types of
counterexamples obtainable from verification or testing tools.

\subsection{Learning vs. Teaching}

We also study the complexity of synthesizing programs from a finite class of programs. This
part of our work is related to previous work on the complexity of {\it teaching}
in exact learning of concepts by Goldman and Kearns~\cite{goldman-td}. 
Informally, the teaching dimension of a concept
class is the minimum number of instances a teacher must reveal to uniquely identify
any target concept from the class. Exact bounds on teaching dimensions
for specific concept classes such as 
orthogonal rectangles, monotonic decision trees, monomials, binary relations
and total orders have been previously presented in literature~\cite{goldman-td,goldman-brto}.
Shinohara et al.~\cite{shinohara-alt90} also introduced a notion of teachability in which a concept class
is teachable by examples if there exists a polynomial size sample under which all consistent
learners will exactly identify the target concept.
Salzberg et al.~\cite{salzberg-td} also consider a model of learning with a helpful teacher. Their model
requires that any teacher using a particular algorithm such as the nearest-neighbor algorithm 
learns the target concept. This work assumes that the teacher knows the algorithm used by the learner.
We do not make any assumption on the inductive learning technique used by the $\ogis$ synthesis
engine. Our goal is to obtain bounds on the number of examples that need to be provided by the oracle to 
synthesize the correct program by relating our framework to the
literature on teaching.

\section{Theoretical Analysis of CEGIS: Preliminaries} \label{sec-not}
\label{sec:not}

Our presentation of formal inductive synthesis and \ogis so far has
not used a particular representation of a concept class or
specification. 
In this section, we begin our theoretical formalization of the counterexample-guided
inductive synthesis (CEGIS) technique, for which such a choice is necessary.
We precisely define the formal
inductive synthesis problem for concepts that correspond to
recursive languages. 
We restrict our attention to the case when the
specification is partial and is a trace property --- i.e.,
the specification is defined by a single formal language. 
This assumption, which is the typical case in formal verification and synthesis, 
also simplifies notation and proofs. Most of our results
extend to the case of more general specifications; we will make
suitable additional remarks about the general case where needed.
For ease of reference, the major definitions and frequently
used notation are summarized 
in Table ~\ref{table:not}.

\subsection{Basic Notation}

We use $\nat$ to denote the set of natural numbers. 
$\nat_i \subset \nat$ denotes a subset of natural numbers $\nat_i = \{ n | n < i\}$. 
Consider a set $S \subset \nat$. $\min(S)$ denotes the minimal element in $S$. 
The union of the sets is denoted by $\cup$ and the intersection of the sets
is denoted by $\cap$. $S_1\setminus S_2$ denotes set minus operation with the resultant set containing
all elements in $S_1$ and not in $S_2$.

\todo{
We denote the set $\nat \cup \{ \bot \}$ as $\natbot$.
A sequence $\sigma$ is a mapping from $\nat$ to $\natbot$.
We denote a prefix of length $k$ of a sequence by $\sigma[k]$.
So, $\sigma[k]$ of length $k$ is a mapping from $\nat_k$ to $\natbot$.
}
$\sigma[0]$ is an empty sequence also denoted by $\sigma_0$ for brevity. 
The set of natural numbers appearing in the sequence $\sigma[i]$ is
defined using a function $\samples$, where 
$\samples(\sigma[i]) = \range(\sigma[i]) - \{\bot\}$.
The set of sequences is denoted by $\Sigma$.


\noindent
{\bf{Languages and Programs:}} We also use standard definitions from computability theory
which relate languages and programs~\cite{rogers-book87}. 
\revise{A set $\lang$ of 
natural numbers is called a computable or recursive  language if there is a program, that is, a 
computable, total function $\algo$ such that for any natural number $n$, 
$$\algo(n) = 1 \; \mathtt{ if } \;  n \in \lang \; \mathtt{ and } \; \algo(n) = 0 \; 
\mathtt{ if } \; n \not \in \lang$$
We say that $\algo$ {\em identifies} the language $\lang$.
Let $\langfun(\algo)$  denote the language $\lang$ 
 identified by the program $\algo$.
The mapping $\langfun$ is not necessarily one-to-one and hence,
syntactically different programs might identify the same language. 
In formal synthesis, we do not distinguish between syntactically different programs
that satisfy the specification.} 
%
Additionally, in this 
paper, we restrict our discussion to recursive languages 
because it includes many interesting and natural classes of languages
that correspond to programs and functions of various kinds, including regular, context free, context
sensitive, and pattern languages.

\revise{
Given a sequence of non-empty languages 
$ \mathscr{\lang} = \lang_0, \lang_1, \lang_2, \ldots$, $\mathscr{\lang}$ is said to be an indexed family of languages 
if and only if for all languages $\lang_i$, 
there exists a recursive function $\template$ such that
$\template(j,n) = \algo(n)$ 
and $\langfun(\algo) = \lang_i$ for some $j$.
Practical applications of program synthesis often consider a family of candidate programs 
which contain syntactically
different programs that are semantically equivalent, that is,
they have the same set of behaviors. 
Formally, in practice program synthesis techniques permit picking $j$ such 
that $\template(j,n) = \algo(n)$ and 
$\langfun(\algo) = L_i$ for all $j \in I_j$ where the set $I_j$ 
represents the
syntactically different but semantically equivalent programs
that produce output $1$ on an input if and only if the
input natural number belongs to $\lang_i$.
Intuitively, a function
$\template$ defines an encoding of the space of candidate programs similar
to encodings proposed in the literature such as those on program sketching~\cite{asplos06} 
and component interconnection encoding~\cite{jha-icse10}.
In the case of formal synthesis where we have a specification $\Spec$, we are
only interested in finding a single program satisfying $\Spec$.
In the general case, $\Spec$ comprises a set of allowed languages, and
the task of synthesis is to find a program identifying some element of this set.
In the case of partial specifications that are trace properties, $\Spec$ comprises
subsets of a single target language $\lang_\target$.
Any program $\algo_\target$ identifying some subset of $\lang_\target$ is a valid solution,
and usually positive examples are used to rule out programs identifying ``uninteresting'' subsets of
$\lang_\target$.
Thus, going forward, we will define the task of program
synthesis as one of identifying the 
corresponding correct language $\lang_\target$.
}

\noindent
{\bf{Ordering of elements in the languages:}} 
A language corresponds to a set of program behaviors. 
We model this set in an abstract manner,
only assuming the presence
of a total order over this set, without prescribing any specific ordering relation. 
Thus, languages are modeled as sets of natural numbers. 
While such an assumption might seem restrictive, we argue that this is
not the case in the setting of CEGIS, where the ordering relation 
is used specifically to model the oracle's preference for returning
specific kinds of counterexamples.
For example, consider the case where elements of a language
are input/output traces.
We can construct a totally ordered set of all possible input/output traces 
using the length of the trace as the primary ordering
metric and the lexicographic ordering as the secondary ordering metric.
Thus, an oracle producing smallest counterexample would produce an 
unique trace which is shortest in length and is lexicographically 
the smallest.
The exact choice of ordering is orthogonal to results presented in our paper, and
using the natural numbers allows us to greatly simplify notation.

\subsection{CEGIS Definitions}

We now specialize the definitions from Sec.~\ref{sec:ogis} for the case of CEGIS.
An indexed family of languages (also called a language class) $\mathscr{\lang}$ defines the concept class $\class$ for synthesis.
The domain $\domain$ for synthesis is the set of natural numbers $\nat$
and the examples are $i \in \nat$.
\revise{Recall that we restrict our attention to the special case
where the specification $\Spec$ is captured by a single target
language, i.e., $\lang_{\target}$ comprising all permitted program behaviors. 
Therefore, the formal inductive synthesis (FIS) problem defined
in Section~\ref{sec:ogis} (Definition~\ref{def:formal-inductive-synthesis}) 
can be restricted for this setting as follows:
\begin{definition}
\acta{
Given a {\em language class} $\mathscr{\lang}$, 
a domain of examples $\nat$, 
the specification $\Spec$ defined by a target language $\lang_\target$, 
and an oracle interface $\orint$, 
the problem of formal inductive synthesis of languages (and the associated
programs) is to identify 
a language in $\Spec$
using
only the oracle interface $\orint$.}
\end{definition}

Counterexample-guided inductive synthesis (CEGIS) is a solution to 
the problem of formal inductive synthesis of languages where
the oracle interface $\orint$ is defined as follows. 

\begin{definition}
\acta{
A {\em counterexample-guided inductive synthesis (CEGIS) oracle interface} is  
$\orint = \defquery \times \defresponse$ where
$\defquery = \{\qposwit, \qceq(\lang)\}$
with $\lang \in \mathcal{\lang}$, 
$\defresponse = \natbot$, and the specification $\Spec$ is 
defined as
subsets of a 
target language $\lang_\target$.
The positive witness query $\qposwit$ returns a positive 
example $i \in \lang_\target$, and 
the counterexample query $\qceq$ takes as argument a candidate language $\lang$
and either returns a counterexample $i\in \lang \setminus \lang_\target$ 
showing that the candidate language $\lang$ is incorrect 
or returns $\bot$ if 
it cannot find a counterexample.}~\footnote{CEGIS techniques in
literature~\cite{asplos06,jin-hscc13} initiate search for correct
program using 
positive examples and use specification to obtain positive examples
corresponding to counterexamples.}

\end{definition}


}

\begin{table}[h!]
\centering
\begin{tabular}{|c | c || c | c|} 
 \hline
 Symbol & Meaning & Symbol & Meaning \\ 
 \hline
 $\nat$ & natural numbers  & $\nat_i$ & natural numbers less than $i$\\ 
 $\min(S)$ & minimal element in set $S$ & $S_1 \setminus S_2$ & set minus \\
 $S_1 \cap S_2$ & set intersection & $S_1 \cup S_2$ & set union \\
 $\sigma$ & sequence of numbers & $\sigma_0$  &  empty sequence \\
 $\sigma[i]$ & sequence of length $i$ & $\sigma(i)$ & $i$th element of sequence $\sigma$ \\
 $\lang_i$ & language (a subset of $\nat$) & $\overline \lang_i$ & complement of language\\
 $\algo_i$ & program for $\lang_i$ & $\langfun(\algo_i) = \lang_i$ & language corresponding to $\algo_i$\\
 $\samples(\sigma)$ & natural numbers in $\sigma$ & $\Sigma$ & set of sequences \\ 
 $\mathscr{\lang}$  & family of languages &  $\mathscr{\algo}$ & family of programs \\
\hline
 $\trace$ & transcript & $\ce$ & counterexample transcript\\
 $\engine$ & synthesis engine & $\learner$ & inductive learning engine\\
 $\verifier_\lang$ & verification oracle for $\lang$ & $\mnverifier_\lang$ &  minimal counterexample oracle \\
 $\bverifier_{B,\lang}$ & bounded counterexample oracle & $\hmnverifier_\lang$ & positive bounded counterexample oracle \\
 $\infcegis$ & \pbox{20cm} {\acta{set of language  families} identified by\\ inf memory cegis engine} & $\cegis$ & \pbox{20cm}{\acta{set of language families} identified by\\  finite memory cegis engine}\\
 $\infmncegis$ & $\infcegis$ with $\mnverifier$ & $\mncegis$ & $\cegis$ with $\mnverifier$\\
 $\infbcegis$ & \pbox{20cm}{ $\infcegis$ with $\bverifier$ for a given\\ constant $B$} & $\bcegis$ & \pbox{20cm} { $\cegis$ with $\bverifier$ for a given\\ constant $B$}\\
 $\infhmncegis$ & $\infcegis$ with $\hmnverifier$ & $\hmncegis$ & $\cegis$ with $\hmnverifier$\\
\hline
\end{tabular}
\caption{Frequently used notation in the paper}
\label{table:not}
\end{table}

\revise{
\acta{
The sequence $\tau$ of responses of the positive witness $\qposwit$
query is called the {\em transcript}, and 
the sequence $\ce$ of the responses to the 
counterexample queries $\qceq$ 
is called the {\em counterexample sequence}.}
The positive witness queries can be answered 
by the oracle sampling examples from the target language.
Our work uses the standard model for language
learning in the limit~\cite{gold67limit}, where the learner has
access to an infinite stream of positive examples from the target language. 
This is also realistic in practical CEGIS settings for infinite concept
classes (e.g.~\cite{jin-hscc13}) where more behaviors can be sampled over time.
We formalize these terms below.
%
}
\begin{definition}
A \todo{{\em transcript}} $\trace$ for a specification language $\lang_\target$ is a sequence with
 $\samples(\trace) = \lang_\target$. $\trace[i]$ denotes the prefix of the \todo{transcript} $\trace$ of length $i$.
 $\trace(i)$ denotes the $i$-th element of the \todo{transcript}.
\label{defn:transcript}
\end{definition}

\revise{

\begin{definition}
\acta{
A {\em counterexample sequence} $\ce$ for a specification language $\lang_\target$ 
from a counterexample query $\qceq$ is a sequence with
$\ce(i) = \qceq(\lang_{cand_i})$, where $\ce[i]$ denotes the prefix of the counterexample sequence $\ce$ of length $i$,
 $\ce(i)$ denotes the $i$-th element of the counterexample sequence,
 and $\lang_{cand_i}$ is the argument of the 
 $i$-th invocation of the query $\qceq$.}
\end{definition}

We now define the verification oracle in CEGIS that produces arbitrary counterexamples, 
as well as its three other variants which generate particular kinds of 
counterexamples. 

\begin{definition}
A {verifier} $\verifier_\lang$ for language $\lang$ is a 
\acta{nondeterministic} mapping from 
$\mathscr{\lang}$ to $\natbot$ such that $\verifier_\lang(\lang_i) = \bot$ if and only if
 $\lang_i \subseteq \lang$, and 
 $\verifier_\lang(\lang_i)  \in \lang_i \setminus \lang$ otherwise.
\end{definition}

\noindent
{\em Remark:} For more general specifications $\Spec$ that are a set of languages, the definition of $\verifier_\lang$ changes in
a natural way: it returns $\bot$ if and only if $\lang_i \in \Spec$ and otherwise returns an example $j$
that is in the intersection of the symmetric differences of each language $\lang \in \Spec$ and the candidate language $\lang_i$.

We define a 
minimal counterexample generating verifier below. The counterexamples are minimal with respect
to the total ordering on the domain of examples. 


\begin{definition}
A  {verifier} $\mnverifier_\lang$ for a language $\lang$ is a 
\acta{nondeterministic}  mapping from 
$\mathscr{\lang}$ to $\natbot$ such that $\mnverifier_\lang(\lang_i) = \bot$ if and only if 
$\lang_i \subseteq \lang$, and
$\mnverifier_\lang(\lang_i) = \min( \lang_i \setminus \lang)$ otherwise.
\end{definition}

Next, we consider another variant of counterexamples, 
namely (constant) bounded counterexamples. 
Bounded model-checking~\cite{biere-bmc} returns a counterexample trace for an incorrect design if it can find 
a counterexample of length less than the specified constant
bound. It fails to find a counterexample for an incorrect design
if no counterexample exists with length less than the given bound.
Verification of concurrent programs by bounding the number of context switches~\cite{qadeer-bounded} 
is another example of the bounded verification technique. This motivates
the definition of a verifier which returns counterexamples bounded by a constant $\bound$.

\begin{definition}
A  {verifier} $\bverifier_{\bound,\lang}$ is a 
\acta{nondeterministic} mapping from 
$\mathscr{\lang}$
to $\natbot$ such that
$\bverifier_{\bound,\lang}(\lang_i) = m$ where $m \in  \lang_i \setminus \lang \wedge m < \bound$ 
for the given bound $\bound$, and $\bverifier_{\bound,\lang}(\lang_i) = \bot$ \acta{if such $m$ does not exist}.
\end{definition}

The last variant of counterexamples is  {\it positive bounded} counterexamples.
The verifier for generating positive bounded counterexample is also provided
with the \todo{transcript} seen so far by the synthesis engine. The verifier generates a counterexample
smaller than the largest positive example in the \todo{transcript}. If there is no counterexample smaller than the
largest positive example in the \todo{transcript}, then the verifier does not return any counterexample.
This is motivated by the practice of mutating correct traces to find bugs in programs and designs.
The counterexamples in these techniques are bounded by the size of positive examples (traces) 
seen so far.\footnote{Note that we can extend this definition to include counterexamples of size bounded
by that of the largest positive example seen so far plus a constant. The proof arguments given in Sec.~\ref{sec:res} 
continue to work with only minor modifications.}
 

\begin{definition}
A {verifier} $\hmnverifier_\lang$ is a \acta{nondeterministic}  
mapping from 
$\mathscr{\lang} \times \Sigma $ to $\natbot$ such that $\hmnverifier_\lang(\lang_i,\trace[n]) = m$ where 
$m \in  \lang_i \setminus \lang \wedge m < \trace(j)$ for some $j \leq n$, and $\hmnverifier_\lang(\lang_i,\trace[n]) = \bot$ 
\acta{if such $m$ does not exist}.
\end{definition}

We now define the oracle for counterexample guided inductive
synthesis. We drop the queries in dialogue since there are only
two kind of queries and instead only use the sequence of responses:
transcript $\tau$ and the counterexample sequence $\ce$. The oracle
also receives as input the current candidate language $\lang_{cand}$ 
to be used as the argument of the $\qcorr$ query.
The overall response of the oracle 
is a pair of elements in $\natbot$. 
\begin{definition}
\acta{An oracle $\deforacle$ for counterexample-guided inductive synthesis
({\em CEGIS oracle}) 
is a nondeterministic 
mapping $\Sigma \times \Sigma \times \mathcal{\lang}  \rightarrow \natbot \times \natbot$
such that 
$\deforacle(\tau[i-1],\ce[i-1],\lang_{cand}) = (\tau(i),\ce(i))$
where $\tau(i)$ is the 
nondeterministic response to positive witness query $\qposwit$  and
$\ce(i)$ is the  nondeterministic response to counterexample query $\qceq(\lang_{cand})$.
The oracle can use 
any of the four verifiers presented earlier to generate the
counterexamples.
An oracle using $\verifier_\lang$ is called $\deforacle_{\cegis}$,
 one using $\mnverifier_\lang$ is called $\deforacle_{\mncegis}$,
 one using $\hmnverifier_\lang$ is called $\deforacle_{\hmncegis}$
and  one using $\bverifier_{\bound,\lang}$ is called $\deforacle_{\bcegis}$.}
\end{definition}

We make the following reasonable assumption on the oracle.
The oracle is assumed to be {\em consistent}: it does not provide the same
example both as a positive example (via a positive witness query) 
and as a negative example (as a counterexample). 
Second, the oracle is assumed to be {\em non-redundant}: it 
does not repeat any positive examples that it may
have previously provided to the learner; for a finite target language,
once the oracle exhausts all positive examples, it will return $\bot$.

The learner is simplified to be a mapping from
the sequence of responses to a candidate program. 
\begin{definition}
An {\em infinite memory learner} $\inflearner$ is a function 
$\Sigma \times \Sigma \rightarrow \mathscr{\lang}$ such that
$\inflearner(\trace[n],\ce[n]) = \lang$
where $\lang$ includes all positive examples in $\trace[n]$
and excludes all examples in $\ce[n]$.\footnote{This holds due to
the specialization of $\Spec$ to a partial specification, and as a trace property. For general
$\Spec$, the learner need not exclude all counterexamples.} 
$\inflearner (\sigma_0, \sigma_0)$ is a predefined constant
representing an initial guess $\lang_0$ of the language, which, for
example, could be $\nat$.
\end{definition}

We now define a finite memory learner which cannot take
the unbounded sequence of responses as argument. The finite
memory learner instead uses the previous candidate program 
to summarize the response sequence. We assume that languages
are encoded in terms of a finite representation 
\acta{(index of the 
language since the language class is an indexed family of languages
and we assume that every index needs unit memory)}
such as a
program that identifies that language.
Such an iterative learner only needs finite memory. 

\begin{definition}
A {\em finite memory learner} $\learner$ is a recursive function 
$\mathscr{\lang} \times \natbot \times \natbot  \rightarrow \mathscr{\lang}$ such that 
for all $n \geq 0$, $\learner(\lang_n,\trace(n),\ce(n)) = \lang_{n+1}$, where
$\lang_{n+1}$ includes all positive examples in $\trace[n]$ and 
excludes all examples in $\ce[n]$. 
We define $\lang_0 = \inflearner (\sigma_0, \sigma_0)$ to be the initial guess of
the language, which for example, could be $\nat$. 
\acta{For ease of presentation,
we omit the finite memory available to the learner in its functional
representation above. The learner can store additional finite information.}
%
\end{definition}

The synthesis engine using infinite memory can now be defined as follows.
\begin{definition}
\acta{
An infinite memory 
$\infcegis$ engine $\engine_{\infcegis}$ is a pair
$\langle \deforacle_{\cegis}, \inflearner \rangle$ comprising 
a CEGIS oracle $\deforacle_{\cegis}$ and an infinite memory learner $\inflearner$,
where, there exists $\tau$ and $\ce$ such that for all $i\geq 0$, 
$\deforacle_{\cegis}(\tau[i],\ce[i],\lang_{i}) = (\trace(i+1), \ce(i+1))$
and $ \lang_{i} = \inflearner(\trace[i],\ce[i])$.
Since the oracle $\deforacle_{\cegis}$ is nondeterministic, 
$\engine_{\infcegis}$  can have multiple transcripts $\tau$ and 
counterexample sequences $\ce$. 
}
%
\end{definition}

\comment{If the oracles are not consistent and the same example is provided as 
a positive example as well as a counterexample, then the learning
engine cannot find any consistent program and it returns the program
corresponding to empty language. The counterexample to empty language by
the subsumption query verifier is $\bot$. So, the overall synthesis engine
produces the program corresponding to empty language $\emptyset$ when
it is not possible to synthesize a program consistent with the output of
oracles.}

A synthesis engine with finite memory cannot store unbounded infinite 
transcripts. 
So, the bounded memory $\cegis$ synthesis engine 
$\engine_{\cegis}$ uses a finite memory learner $\learner$.
\comment{maintains its current hypothesis program and some bounded memory at every iteration. 
An inductive learning function $\learner$ uses bounded memory and 
takes the current hypothesis program, an example
and a counterexample, and produces the corresponding next hypothesis program. 
$\learner$ can access 
additional state $\scratch$ stored by the synthesis engine in the finite bounded memory.}

\begin{definition}
\acta{
A finite memory 
$\cegis$ engine $\engine_{\cegis}$ is a tuple
$\langle \deforacle_{\cegis}, \learner \rangle$ comprising 
a CEGIS oracle $\deforacle_{\cegis}$ and a finite memory learner $\learner$ 
where, there exists $\tau$ and $\ce$ such that for all $i\geq0$,   
$\deforacle_{\cegis}(\tau[i],\ce[i],\lang_{i+1}) = (\trace(i+1), \ce(i+1))$
 and
$\lang_{i} = \learner(\lang_i, \trace(i),\ce(i)) $.
Since the oracle $\deforacle_{\cegis}$ is nondeterministic, 
$\engine_{\infcegis}$  can have multiple transcripts $\tau$ and counterexample
sequences $\ce$. }
\end{definition}
}

\comment{
\begin{definition}
A finite memory 
$\cegis$ engine $\engine_{\cegis} : \Sigma \times \Sigma \rightarrow \mathscr{\algo}$ is defined recursively 
below.
$$\engine_{\cegis} (\trace[n], \ce[n]) = \learner( \engine_{\cegis} (\trace[n-1], \ce[n-1]),  \trace(n), \ce(n) , \scratch) $$
where $\learner$ is a recursive function 
$\mathscr{\algo} \times (\nat \cup \{\bot\}) \times (\nat \cup \{\bot\})  \times \Scratch \rightarrow \mathscr{\algo}$ 
that characterizes the learning engine and how it eliminates counterexamples,
$\scratch \in \Scratch$ is a finite state that might be maintained by the
learner beyond maintaining the current hypothesis, $\trace[n]$ is a 
transcript for language 
$\lang$ and $\ce$ is a counterexample sequence such that
$$\ce(i) = \verifier_\lang(\langfun(\engine_{\cegis}(\trace[i-1], \ce[i-1]))) $$
$\engine_{\cegis} (\sigma_0, \sigma_0)$ is a predefined constant representing an initial guess 
$\algo_0$ of the program, which, for example, could be program corresponding to the universal language $\nat$.
\end{definition}

Intuitively, $\cegis$ is provided with a transcript along with a counterexample
transcript formed by counterexamples to the latest conjectured languages from the verifier $\verifier_\lang$.
}

\acta{A pair $(\tau,\ce)$ is a valid transcript and counterexample
sequence for $\engine_{\cegis}$ if the above definitions hold for that pair.
We denote this by $(\trace,\ce)\models \engine_{\cegis}$.}
Similar to Definition~\ref{def:ogis}, the convergence of 
the counterexample-guided synthesis engine is defined as follows:
\begin{definition}
\acta{
We say that $\engine_{\cegis}: \langle \deforacle_{\cegis}, \learner \rangle$ identifies $\lang$, that is, it converges to $\lang$, 
written $\engine_{\cegis} \rightarrow \lang$}  if and only if
there exists $k$ such that for all $n\geq k$, 
\acta{$\learner(\lang_n, \trace[n], \ce[n]) =  \lang$
for all valid transcripts $\tau$ and counterexample sequences $\ce$ of 
$\engine_{\cegis}$}. 
\label{defn:convergence}
\label{defn:identifies-lang}
\end{definition}
This notion of convergence is standard in language
learning in the limit~\cite{gold67limit}. For the case of general specifications $\Spec$,
as given in Definition~\ref{def:ogis}, the synthesizer must converge to {\it some} language in $\Spec$.
\comment{\begin{definition}
$\engine_{\cegis}$ identifies a language $\lang$ if and only if 
for all transcripts $\tau$ for the specification language 
$\engine_{\cegis}$ converges to $\lang$,
that is, 
$\forall \tau \;\; (\engine_{\cegis} \xrightarrow{\tau} \lang)$.
\label{defn:identifies-lang}
\end{definition}
}
As per Definition~\ref{defn:transcript}, a transcript is an infinite sequence of examples  
which contains all the elements in the target language. Definition~\ref{defn:convergence}
requires the synthesis engine to converge to the correct
language after consuming a {\em finite} part of the transcript
and counterexample sequence. This notion of
convergence is standard in the literature on 
language learning in the limit~\cite{gold67limit}~\footnote{In this framework, a synthesis engine
is only required to converge to the correct concept without requiring it to 
recognize it has converged and terminate. For a finite concept or language, 
termination can be trivially guaranteed when the oracle is assumed to be non-redundant and does not repeat examples. 
}.

We extend Definition~\ref{defn:identifies-lang} to general specifications $\Spec$ as follows:
$\engine_{\cegis}$ identifies a specification $\Spec$ if it 
identifies {\it some} language in $\Spec$.
As noted before, this section focuses on the case of a partial specification that is a trace property. 
In this case, $\Spec$ comprises all subsets of a target language $\lang_\target$.
Since Definition~\ref{defn:transcript} defines a transcript as comprising all positive examples
in $\lang_\target$ and Definition~\ref{defn:identifies-lang} requires convergence for all possible
transcripts, the two notions of identifying $\Spec$ and identifying $\lang_\target$ coincide.
We therefore focus in Sec.~\ref{sec:res} purely on language identification with the observation that our results
carry over to the case of ``specification identification''.

\begin{definition}
\acta{$\engine_{\cegis} = \langle \deforacle_{\cegis}, \learner \rangle$} identifies a language family $\mathscr{\lang}$ 
if and only if 
$\engine_{\cegis}$ identifies every language $\lang \in \mathscr{\lang}$.
\end{definition}

The above definition extends to families of specifications in an exactly analogous manner.
We now define the set of language families that can be identified by the 
inductive synthesis engines as $\cegis$ formally below. 


\begin{definition}
\acta{$\cegis = \{ \; \mathcal{L} \; | \; \exists \learner \; 
\forall \deforacle_{\cegis} \; . \; $ the engine $\engine_{\cegis} = \langle \deforacle_{\cegis}, \learner \rangle$ identifies $\mathscr{\lang} \}$.}\\ 
\end{definition}

\comment{

\begin{definition}
\acta{$\cegis = \{ \; \mathcal{L} \; | \; \exists \learner \; $
such that  $  \engine_{\cegis} = \langle \deforacle_{\cegis}, \learner \rangle$ identifies $\mathscr{\lang} \}$.}\\ 
To elaborate, 
$\cegis = \{ \; \mathcal{L} \; | \; \exists \learner \; \forall (\trace,\ce)$
such that  $  \engine_{\cegis} = \langle \deforacle_{\cegis}, \learner \rangle$ and $(\trace,\ce) \models \engine_{\cegis}$, 
the synthesis engine 
$\engine_{\cegis}$ identifies every language in  $\mathscr{\lang} \}$.
\end{definition}
}


\acta{The convergence of synthesis engine to the correct language, 
identification condition for a language, and 
language family identified by a synthesis engine
are defined similarly 
as listed in 
Table~\ref{tbl:cegisvariants}.}

\begin{table}[h!]
\centering
\begin{tabular}{|c | c | c | c | c |} 
 \hline
 {\small {Learner /\ Oracle}} &  $\deforacle_{\cegis}$ &  $\deforacle_{\mncegis}$ &  $\deforacle_{\hmncegis}$\\
 \hline
 Finite memory  $\learner$ & $ \engine_\cegis, \cegis$ & $\engine_\mncegis, \mncegis$ & $\engine_\hmncegis, \hmncegis$\\
  Infinite memory  $\inflearner$ & $ \engine_\infcegis, \infcegis$ & $\engine_\infmncegis, \infmncegis$ & $\engine_\infhmncegis, \infhmncegis$ \\
 \hline
\end{tabular}
\caption{Synthesis engines and corresponding sets of language families}
\label{tbl:cegisvariants}
\end{table}

\acta{
The constant bounded counterexample-guided inductive synthesis oracle
$\deforacle_{\bcegis}$ 
uses the verifier $\bverifier_{\bound,\lang}$. It takes 
an additional parameter $\bound$ which is the
constant bound on the maximum
size of a counterexample. If the verifier cannot find a counterexample
below this bound, it will respond with $\bot$. 

\begin{definition}
\acta{Given a bound $\bound$, $\engine_{\bcegis} = \langle \deforacle_{\bcegis}, \learner \rangle$}
where $\deforacle_{\bcegis}$ uses $\bverifier_{\bound,\lang}$,  
we say that $\engine_{\bcegis}$ 
identifies a language family $\mathscr{\lang}$ 
if and only if 
$\engine_{\bcegis}$ identifies every language $\lang \in \mathscr{\lang}$.
\end{definition}

Note that the values of $\bound$ for which a language family $\mathscr{\lang}$ is identifiable 
can be different for different $\mathscr{\lang}$. 
The overall class of language families 
identifiable using $\deforacle_{\bcegis}$ oracles can thus be defined as follows:

\begin{definition}
\acta{
$\bcegis = \{ \; \mathcal{L} \; | \; \exists \bound \; \; \exists \learner \; . \;
\forall \deforacle_\bcegis \; s.t. \deforacle_\bcegis$ uses $\bverifier_{\bound,\lang} \; \; . \;$ 
the engine $\engine_{\cegis} = \langle \deforacle_\bcegis, \learner \rangle$ identifies $\mathscr{\lang} \}$ 
}
\end{definition}

\comment{
\begin{definition}
\acta{
$\bcegis = \{ \; \mathcal{L} \; | \; \exists \bound \; \exists \learner \;$ such that $\engine_{\cegis} = \langle \deforacle_\bcegis, \learner \rangle$ identifies $\mathscr{\lang} \}$ where
$\deforacle_\bcegis$ uses $\bverifier_{\bound,\lang}$}.\\
To elaborate, 
$\bcegis = \{ \; \mathcal{L} \; | \; \exists \bound \exists \learner \; \forall (\trace,\ce)$
such that  $  \engine_{\cegis} = \langle \deforacle_{\bcegis}, \learner \rangle$ and $(\trace,\ce) \models \engine_{\bcegis}$, 
the synthesis engine 
$\engine_{\bcegis}$ identifies every language in  $\mathscr{\lang} \}$.
\end{definition}
}
}

\section{Theoretical Analysis of CEGIS: Results}
\label{sec-res}
\label{sec:res}

In this section, we present the theoretical results when the class of languages 
(programs) is infinite.
\revise{ We consider two axes of variation. 
We first consider the case in which the inductive learning
technique has finite memory in Section~\ref{sec-resfin}, 
and then the case in which it has
infinite memory in Section~\ref{sec-resinf}.  
For both cases, we consider the four kinds of counterexamples mentioned in Section~\ref{sec-intro}
and Section~\ref{sec-not};
namely, arbitrary counterexamples, minimal counterexamples, constant bounded counterexamples and 
positive bounded counterexamples. }

For simplicity, our proofs focus on the case of partial specifications that are trace properties,
the common case in formal verification and synthesis. Thus, $\Spec$ comprises subsets of a target
specification language $\lang_\target$.
However, many of the results given here 
extend to the case of general specifications. Most of our theorems
show differences between language classes for CEGIS variants --- i.e., theorems showing that there is a specification on which
one variant of CEGIS converges while the other does not --- and for these, it suffices to show such a difference for 
the more restricted class of partial specifications.
The results also extend to the case of equality between language classes (e.g., Theorem~\ref{thm-min})
in certain cases; we make
suitable remarks alongside.

\subsection{Finite Memory Inductive Synthesis}
\label{sec-resfin}

We investigate the four language classes $\cegis, \mncegis, \bcegis$ and  $\hmncegis$ identified by
the synthesis engines  $\engine_{\cegis}$, $\engine_{\mncegis}$, $\engine_{\bcegis}$ and $\engine_{\hmncegis}$
and establish relations between them. We show that $\bcegis \subset \mncegis = \cegis$, 
$\hmncegis \not \subseteq \cegis$ and $\hmncegis \not \supseteq \cegis$.

\subsubsection{Minimal vs. Arbitrary Counterexamples}

We begin by showing that replacing a deductive
verification engine which returns arbitrary counterexamples with a deductive verification engine which 
returns minimal counterexamples does not change the power of counterexample-guided inductive synthesis.
The result is summarized in Theorem~\ref{thm-min}.

\begin{theorem} \label{thm-min}
The power of synthesis techniques using arbitrary counterexamples and those using minimal counterexamples are 
equivalent, that is, $\mncegis = \cegis$.
\end{theorem}

\begin{proof}
$\mnverifier_\lang$ is a special case of $\verifier_\lang$ in that a minimal counterexample
reported by $\mnverifier_\lang$ can be treated as arbitrary counterexample to simulate \acta{$\engine_\cegis$ using 
$\engine_\mncegis$}. 
Thus, $\cegis \subseteq \mncegis$.

The more interesting case to prove is $\mncegis \subseteq \cegis$. For a language $\lang$, 
let $\mncegis$ converge to the correct 
language $\lang$ on  \todo{transcript} $\trace$. 
We show that $\engine_{\cegis}$ can simulate $\engine_{\mncegis}$ and also converge to $\lang$
on \todo{transcript} $\trace$.
The proof idea is to show that a finite learner can simulate
$\mnverifier_\lang$ by making a finite number of calls to $\verifier_\lang$. 
Therefore, the learner sees the same counterexample sequence with $\verifier_\lang$
as with $\mnverifier_\lang$ and thus converges to the same language in both cases.

{\color{black} 
Consider an arbitrary step of the dialogue between learner and verifier
when a counterexample is returned.
Let the arbitrary counterexample returned by the verifier
for a candidate language $\lang_i$ 
be $c$, that is $\verifier_\lang(\lang_i) = c$. 
Thus, $c$ is an upper bound on the minimal counterexample returned by $\mnverifier_\lang$. 
The latter can be recovered using the following characterization:
$$\mnverifier_\lang( \lang_i) = \text{ minimum } j \text{ such that }  \verifier_\lang(\{j\}) \text{ is not } \bot \text{ for } 0 \leq j \leq \verifier_\lang(\lang_i) $$
The learner can thus perform at most $c$ queries
to $\verifier_\lang$ to compute the minimal counterexample that would be
returned by $\mnverifier_\lang$. In case of totally ordered set (such as $\nat$),
we could do this more efficiently using binary search. 
At each stage of the iteration, 
the learner needs to store the smallest counterexample
returned so far. 
Thus, the work performed by the learner
in each iteration to craft queries to $\verifier_\lang$ 
can be done with finite memory.
$\mnverifier_\lang(\lang_i)$ can be computed
using finite memory and using at most $c= \verifier_\lang(\lang_i)$ calls of \acta{$\verifier_\lang$}. 
} 

\comment{
Then, $\verifier(\lang \cap \{c\}, \lang_i \cap \{c\}) \not = \bot$ 
and hence, $c$ is upper bound on the number of 
times that $\verifier$ needs to be called in finding the minimum 
counterexample. In other words,
$$\mnverifier(\lang, \lang_i) = \text{ minimum } j \text{ such that }  \verifier(\lang \cap \{j\}, \lang_i \cap \{j\}) \text{ is not } \bot \text{ for } 0 \leq j \leq \verifier(\lang, \lang_i) $$
 We only need finite memory to 
store the current value of  $j$ while iteratively calling the verifier 
$\verifier$.  Thus, $\mnverifier(\lang,\lang_i)$ can be computed
using finite memory and using at most $c= \verifier(\lang, \lang_i)$ calls of $\verifier$. 
}

Thus, $\engine_{\cegis}$ can simulate $\engine_{\mncegis}$
by finding the minimal counterexample at each step using the verifier
$\verifier$ iteratively as described above. 
This implies that $\mncegis = \cegis$.
\qed

\end{proof}

Thus, $\mncegis$ successfully converges to the correct language 
 if and only if
$\cegis$ also successfully converges to the correct language. So, there is no increase or decrease
in power of synthesis by using the deductive verifier that provides minimal counterexamples.

{\em Remark:} The above result (and its analog in Sec.~\ref{sec-resinf}) 
also holds in the case of general specifications when CEGIS is
replaced by Generalized CEGIS. In particular, if either crafted correctness ($\qccorr$) or membership queries ($\qmem$)
are introduced, then it is easy to show that $\cegis$ can simulate $\mncegis$ by mimicking each
step of $\mncegis$ by recovering the same counterexample it used with suitable $\qmem$ or $\qccorr$ queries.
In this case, $\cegis$ can converge to every language that $\mncegis$ converges to, and hence identifies
the same class of specifications.


\subsubsection{Bounded vs. Arbitrary Counterexamples}

We next investigate $\bcegis$ and compare its relative synthesis power  to $\cegis$.
As intuitively expected, $\bcegis$ is strictly less powerful than $\cegis$ as summarized
in Theorem~\ref{thm-b} which formalizes the intuition. 

\begin{theorem}
The power of synthesis techniques using bounded counterexamples 
is less than those using counterexamples, that is, $\bcegis \subset \cegis$.
\label{thm-b}
\end{theorem}

\begin{proof}
Since bounded counterexample is also a counterexample, we can easily simulate a bounded verifier
$\bverifier$ using a $\verifier$ by ignoring counterexamples from $\verifier$ if they are larger
than a specified bound $\bound$ which is a fixed parameter and can be stored in the finite memory
of the inductive learner. Thus, $\bcegis \subseteq \cegis$. 

We now describe a language class  for which the corresponding languages
cannot be identified using bounded counterexamples.
\begin{langfamily}
: $\mathcal{\lang}_{notcb} =\{ \lang_i |  
\todo{ i > \bound \text{ and } \lang_i = \{ n | n \in \nat \wedge n > i \}}
\}$ where
$\bound$ is a constant bound.
\label{lf:cbweak}
\end{langfamily}
\acta{We provide this by contradiction. Let us assume that there 
is a $\engine_\bcegis$ that 
can identify languages in 
 $\mathcal{\lang}_{notcb}$. Let the verifier used by $\engine_\bcegis$ 
 be $\bverifier$ and 
$\bound'$ be the constant bound on the counterexamples produced by
 $\bverifier$.
Let us consider the languages
$\mathcal{\lang}_{notcbfail}  =\{ \lang_j |  \lang_j \in \mathcal{\lang}_{notcb} \wedge j> B'  \} \subseteq  \mathcal{\lang}_{notcb}$, the set of counterexamples 
that can be produced by
$\bverifier$ is the same for these languages
(that is, $\{ n | n \in \nat \wedge n \leq B' \}$) 
since the counterexamples
produced by $\bverifier$ cannot be larger than $B'$. 
Thus, a synthesis engine $\engine_\bcegis$ cannot
 distinguish between languages in $\mathcal{\lang}_{notcbfail}$ which is
 a contradiction.
 Thus, $\engine_\bcegis$ cannot identify all languages in 
 $\mathcal{\lang}_{notcb}$. 
 $\engine_\cegis$ can identify all languages in  $\mathcal{\lang}_{notcb}$
 using a simple learner which proposes $L_i$ as the hypothesis
 language if $i$ is the smallest positive example seen so far. So, $\bcegis \subset \cegis$.
 }
\qed
 
\end{proof}

We next analyze $\hmncegis$,  and show that it is not equivalent to $\cegis$ or contained in it.
So, replacing a deductive verification engine
which returns arbitrary counterexamples with a verification engine which returns
counterexamples bounded by history of positive examples has impact on the power of the synthesis technique.
But this does not strictly increase the power of synthesis. Instead, the
use of positive history bounded counterexamples allows languages from
new classes to be identified but at the same time, language from some
language classes which could be identified by 
$\cegis$ can no longer be identified using positive bounded counterexamples.
The main result regarding the power of synthesis techniques using
 positive bounded counterexamples is summarized in Theorem~\ref{thm-hb}.

\begin{theorem} \label{thm-hb}
The power of synthesis techniques using arbitrary counterexamples and those using
positive bounded counterexamples are not equivalent, and none is more powerful than the other.
$\hmncegis \not = \cegis$. In fact,  $\hmncegis \not \subseteq \cegis$ and $\cegis \not \subseteq \hmncegis$.
\end{theorem}

We prove this using the following two lemmas. The first lemma \ref{lemma-cegisbetter}
shows that there is a family of
languages from which a language can be identified by $\cegis$ but,
this cannot be done by $\hmncegis$. The second lemma \ref{lemma-hcegisbetter} shows that
there is another family of languages from which a language can be identified by
$\hmncegis$ but not by $\cegis$.

\begin{lemma} \label{lemma-cegisbetter}
There is a family of languages $\mathcal{L}$ such that 
$\hmncegis$ cannot \acta{identify every language 
$\lang$ in $\mathcal{L}$ but
$\cegis$ can do so}, that is,
 $\cegis \not \subseteq \hmncegis$.
\end{lemma}

\begin{proof}
Now, consider the language family~\ref{lf:pbweak} formed by upper bounding the elements by some fixed constant. Let the target language $\lang$ (for which we want to identify $\lang_i$. In rest of the proof, we also refer to this family as $\mathcal{L}$ for brevity.
\begin{langfamily}
$\mathcal{L}_{notpb} = \{ \lang_i | i \in \nat \}$ such that 
$\lang_i = \{ n | n \in \nat \wedge n \leq i \}$.
\label{lf:pbweak}
\end{langfamily}

If we obtain a \todo{transcript} $\trace[j]$ at any point in synthesis
using positive bounded counterexamples, then for any intermediate language $\lang_j$ proposed by
$\engine_{\hmncegis}$, $\hmnverifier_\lang$  would always return $\bot$ since all the counterexamples
would be larger than any element in $\trace[j]$. This is the consequence of the chosen languages
in which all counterexamples to the language are larger than any positive example of the language.
So, $\engine_{\hmncegis}$ cannot identify the target language $\lang$.

But we can easily design a synthesis engine $\engine_{\cegis}$
using arbitrary counterexamples that can synthesize
$\algo$ corresponding to the target language $\lang$. The algorithm starts with $L_0$ as its
initial guess. If there is no counterexample, the algorithm next guess is $\lang_1$.
In each iteration $j$,
the algorithm guesses $\lang_{j+1}$ as long as there are no counterexamples. When a counterexample
is returned by $\verifier_\lang$ on the guess $\lang_{j+1}$,
the algorithm stops and reports the previous guess $\lang_j$ as the correct language.

Since the elements in each language $\lang_i$ is bounded by some fixed constant $i$, the above synthesis procedure $\engine_{\cegis}$ is guaranteed to terminate after $i$ iterations when identifying any language $\lang_i \in \mathcal{L}$. Further, $\verifier_\lang$ did not return any counterexample up to iteration $j-1$ and so, $\lang_j \subseteq \lang_i$. And in the next iteration, a counterexample was generated.
So, $\lang_{j+1} \not \subseteq \lang_i$. Since, the languages in $\mathcal{L}$ form a monotonic chain
$\lang_0 \subset \lang_1 \ldots $. So, $\lang_j = \lang_i$. In fact, $j=i$ and in the $i$-th iteration,
the language $L_i$ is correctly identified by $\engine_{\cegis}$.
Thus,  $\cegis \not \subseteq \hmncegis$.
\qed

\end{proof}

This shows that $\cegis$ can be used to identify languages when $\hmncegis$ will fail. Putting a
restriction on the verifier to only produce counterexamples which are bounded by the
positive examples seen so far does not strictly increase the power of synthesis.

We now show that this restriction enables identification 
of languages which cannot be identified by $\cegis$. 

In the proof below, we construct a language which is not distinguishable
using arbitrary counterexamples and instead, it relies on the verifier keeping a record of the
largest positive example seen so far and restricting counterexamples to those below the largest positive
example.

\begin{lemma} \label{lemma-hcegisbetter}
There is a family of languages $\mathcal{L}$ such that, $\cegis$ cannot identify a  language $\lang$ in $\mathcal{L}$ but
$\hmncegis$ can identify  $\lang$, that is,
 $\hmncegis \not \subseteq \cegis$.
\end{lemma}

\begin{proof}
Consider the language 
$$\lang^{32} = \{ 3^j.2^i | j \in \{0,1\}, i \in \nat \}  $$ 
\acta{ where $3^j.2^i$ is a natural number obtained by taking the
product of $3$ raised to the power of $j$ and $2$ raised to the power
of $i$. $\lang^{32}$ is a set of these natural numbers.} 
We now construct a family of languages which are finite subsets of $\lang^{32}$ and have 
at least one member of the form $3.2^i$, that is,
$$\mathcal{L}^{32} = \{\lang^{32}_i | i \in \nat, \lang^{32}_i \subset \lang^{32}, \lang^{32}_i \text{ is finite } \text{ and } \exists k \; s.t. 
\; 3.2^k  \in \lang^{32}_i \}$$
We now consider the language
$$\lang^2 = \{ 2^i | i \in \nat\}$$ 
Now, let $\mathcal{\lang}^{2}$ be the family of languages such that the smallest element member in the language 
is the same as the index of the language, that is,
$$\mathcal{\lang}^{2} = \{ \lang^2_i | i \in \nat,  \lang^2_i \subseteq \lang^2, \lang^2_i \text{ is infinite and } \min(\lang^2_i) = 2^i \}$$

Now, we consider the following family of languages below.
\begin{langfamily} $$ \mathcal{L}_{pb} = \mathcal{L}^{32} \cup \mathcal{L}^{2}$$
\label{lf:pb} 
\end{langfamily}
We refer to this language as $\mathcal{L}$ in rest of the proof for brevity.
We show that there is a language $\lang$ in $\mathcal{L}$ such that the 
language 
$\lang$ cannot be identified by $\cegis$ but $\hmncegis$ can 
identify 
any language in $\mathcal{L}$.

The key intuition is as follows.
If the examples seen by synthesis algorithm till some iteration $i$ are all of the form 
$2^j$, then any synthesis technique cannot differentiate whether the
language belongs to $\mathcal{L}^{32}$ or $\mathcal{L}^{2}$. If the language belongs
to  $\mathcal{L}^{32}$, the synthesis engine would eventually obtain
an example of the form  $3.2^j$ (since each language in $\mathcal{L}^{32}$
has at least one element of this kind and these languages are finite). While
the synthesis technique using arbitrary counterexamples cannot recover the previous examples,
the techniques with access to the verifier which produces positive bounded counterexamples
can recover all the previous examples.

We now specify a $\engine_\hmncegis$ which can identify languages in $\mathcal{L}$. The synthesis approach works in two possible steps.
\begin{itemize}[leftmargin=*]
\item Until an example $3.2^j$ is seen by the synthesis engine, let $2^i$ be the smallest member element
seen so far in the \todo{transcript}, the learner proposes $\lang_i$ as the language. If the target
language $\lang \in \mathcal{L}^{2}$, the learner would eventually identify the language since
the minimal element will show up in the  \todo{transcript}.
If the target language $\lang \in \mathcal{L}^{32}$, then eventually, an example of the form
$3.2^j$  will be seen since $\lang$ must have one such member element. And after such an example 
is seen in the \todo{transcript}, the synthesis engine moves to second step.
\item After an example of the form $3.2^j$ is seen, the synthesis engine can now be sure that the language
belongs to $\mathcal{L}^{32}$ and is finite. Now, the learner can discover all the positive examples 
seen so far
using the following trick. We first discover the upper bound $B_p$
on positive examples seen so far.
{\color{black} 
$$B_p = \text{ minimum k such that } \hmnverifier_\lang(\{3^k\}, \trace[n]) \text{ returns }  \bot \text{ for } k = 2,3,\ldots$$  
Recall that $3^k, k = 2,3, \ldots$ are not in the target language
since they are not in any of the languages in the $\mathcal{L}$ to which the target language belongs.
$\hmnverifier_\lang$ will return the only element $3^k$ in the proposed
candidate language as a counterexample as long as
there is some positive example $2^i$ seen previously such that $2^i \geq  3^k$. So, $3^{B_p}$ is the upper bound on all the positive examples seen
 so far.
The learner can now construct singleton languages 
$\{2^j\}$ for $j=0,1,... l$ such that $2^l < 3^{B_p}$. 
If a counterexample is returned 
by $\hmnverifier_\lang(\{2^i\}, \trace[n])$ 
then $2^i$ is not in the target language.
If no counterexample is returned, then $2^i$ is in the target language. This allows the synthesis engine
to recover all the positive examples seen previously in finite steps.
} 
As we recover the positive examples, we run a Gold style
algorithm for identifying finite languages~\cite{jain07} 
to converge to the correct language.
Thus, the learner would identify the correct language using finite memory.   
\end{itemize}

We now prove that $\cegis$ does not identify this family of languages. Let us
assume that $\mathcal{L} \in \cegis$. So, there is a synthesis engine $\engine_\cegis$
which can identify all languages in $\mathcal{L}$. So, $\engine_\cegis$ must converge to any language $\lang_1 \in \mathscr{L}^2$ after some finite
transcript $\tau_s$. Let us consider an extension $\tau_s 2^m$ of $\tau_s$ 
such that $2^m \in \lang_1$ and $2^m \not \in \samples(\tau_s)$. Such 
an element $2^m$ exists 
since $\tau_s$ is a finite transcript and $\lang_1$ is an
infinite language. Since the learner converges to $\lang_1$
starting from the initial language $\lang_0$ 
after consuming $\tau_s$, $\learner(\lang_0,\tau_s2^m, \ce') = \learner(\lang_0, \tau_s, \ce)$.

Let us consider two transcripts $\tau_s 2^m (3.2^p)\bot^\omega$ and 
$\tau_s (3.2^p)\bot^\omega$ where $\bot^\omega$ denotes repeating
$\bot$ infinitely in the rest of the transcript.
We know that $\learner(\lang_0,\tau_s 2^m, \ce') = \learner(\lang_0, \tau_s, \ce) = \lang_1$
and thus, $\learner(\tau_s 2^m (3.2^p) \bot^\omega, \ce') =  \learner(\tau_s (3.2^p) \bot^\omega, \ce) = \learner(\lang_1, (3.2^p) \bot^\omega, \ce'') $. So, the synthesis engine would behave exactly 
the same for both
transcripts, and if it converges to a language $\lang_2$ on one transcript, it would converge to the same language on the other transcript. But the two
transcripts are clearly from two different languages in $\mathcal{L}^{32}$.
One of the transcripts corresponds to
the finite language $\samples(\trace_s) \cup \{3.2^p\}$
and the other corresponds to $\samples(\trace_s) \cup \{2^m,3.2^p\}$.
This is a contradiction and hence, there is no synthesis engine using
arbitrary counterexamples
$\engine_{\cegis}$ that can identify all languages in $\mathcal{L}$. 

\qed

\end{proof}

\subsubsection{Different Flavors of Bounded Counterexamples}

Finally, we compare $\hmncegis$ and $\bcegis$ and show that they are not contained in each other.

\begin{theorem}
The power of synthesis techniques using bounded counterexamples 
is neither less nor more than the techniques using positive bounded counterexamples, 
that is, $\bcegis \not \subseteq \hmncegis$ and  $\hmncegis \not \subseteq \bcegis$.
\label{thm-bhb}
\end{theorem}

\begin{proof}
We consider two languages considered in previous proofs and show that the
languages corresponding to one of them
can only be identified by $\hmncegis$ while the 
languages corresponding to the other
can only be identified by $\bcegis$.

Consider the language family~\ref{lf:cbweak} ($\mathcal{L}_{notcb}$)
 formed by lower bounding the elements by some fixed constant, that is,
\acta{
 $\mathcal{\lang}_{notcb} =\{ \lang_i |  
\todo{ i > \bound \text{ and } \lang_i = \{ n | n \in \nat \wedge n > i \}}
\}$ where
$\bound$ is a fixed integer constant. 
We have proved in Theorem~\ref{thm-b} that a synthesis engine
$\engine_\bcegis$ 
cannot identify all languages in 
 $\mathcal{\lang}_{notcb}$.
 On the other hand, 
 any counterexample is smaller than all positive examples
 in any language in $\mathcal{\lang}_{notcb}$. 
 So, a verifier
producing positive bounded counterexample behaves similar to an arbitrary counterexample verifier since any positive example is larger than 
all negative examples.
Thus, $\engine_\cegis$ can identify languages in this language class.
 So, $\hmncegis \not \subseteq \bcegis$.}
 
Now, consider the family of languages consisting of these, that is,
\begin{langfamily}
$\mathcal{L}_{cbnotpb} = \{ \lang_i | i < \bound \}$ where $\lang_i = \{ n | n \in \nat \wedge n \leq i \}$
\label{lf:cbnotpb}
\end{langfamily}
This is a slight variant of the language class considered in proving $\engine_\cegis$
to be more powerful than $\engine_\hmncegis$ where we have restricted the class of languages
to be a finite set. As stated earlier, $\hmnverifier$ does not produce 
any counterexample for these languages since all positive examples are smaller than
any counterexample. But $\bverifier$ can be used to identify 
languages in this class by selecting the bound of the counterexamples 
to be $\bound$. Since, the counterexamples
are at most of size $\bound$ for these languages, a bounded counterexample verifier 
behaves exactly like an arbitrary counterexample producing verifier.
Thus, $\bcegis \not \subseteq \hmncegis$.
\qed
\end{proof}

\subsection{Infinite Memory Inductive Synthesis}
\label{sec-resinf}

We now consider the case where the inductive learning engine has infinite unbounded memory. This case
is simpler than the one considered earlier with finite memory bound on the inductive learning engine
and 
\acta{most of the results presented here follow from the results proved for the finite memory case. For brevity of space, we only give proof sketches
highlighting the difference from the finite memory case.
\begin{enumerate}
\item The proof of Theorem~\ref{thm-min} works even when we replace the inductive learning engine
using finite memory with the one using infinite memory. 
Further, the minimal counterexample can still be
used as an aribitrary counterexample. And so, $\infmncegis = \infcegis$.

\item Next, we show that $\infbcegis \subseteq \infcegis$.
Consider an arbitrary but fixed constant $\bound$. For this $\bound$,
consider all verifiers $\bverifier$ that only produce counterexamples bounded
by $\bound$. We wish to argue that any infinite memory learner
$\inflearner$ that can converge to a target language $\lang_{\target}$
using any $\bverifier$ can also do so using $\verifier$. 
The basic idea is as follows: since $\inflearner$ has infinite memory,
it can make extra queries to $\verifier$ to obtain counterexamples
bounded by $\bound$ and learns only from those. 
Suppose at some step it received a counterexample
$x$ bigger than $\bound$ for candidate language $\lang$. Then
$\inflearner$ constructs a new candidate language $\lang'$ that excludes $x$
but otherwise agrees with $\lang$.\footnote{We can do this as we have
a finite representation of $\lang$ (e.g., in the form 
of its characteristic function) and can modify this to initially check
if the input is $x$, and if so, to report that this is not in the
modified language.} 
It then queries $\verifier$ with this new candidate $\lang'$, and
iterates the process until a counterexample less than $\bound$ is
received (which must happen if such a counterexample exists).
$\inflearner$ uses its infinite-size memory to construct candidate languages
that keep track of a potentially unbounded number of counterexamples
bigger than $\bound$.
Thus, $\inflearner$ uses this procedure to convert any $\verifier$ into
some $\bverifier$. Since $\infbcegis$ comprises all language families
learnable by $\inflearner$ given {\it any} $\bverifier$, these language
families are also learnable by $\inflearner$ using $\verifier$.
Therefore, $\infbcegis \subseteq \infcegis$.

\item 
We now sketch the proof for $\infhmncegis \subseteq \infcegis$.  The
argument is similar to the previous case. 
Since the learner has infinite memory, it can
store all the positive examples seen so far. Moreover, similar to the
case of $\infbcegis$, it can construct a stream of candidate languages
to query $\verifier$ so as to obtain positive history bounded
counterexamples, as follows.
It queries $\verifier$
to obtain an arbitrary counterexample. 
If this is smaller than the 
largest positive example in stored positive examples, then
the learner uses this example for proposing the next hypothesis language. 
If this counterexample is larger that the largest positive example, 
it constructs a new candidate language by excluding this
counterexample from the previous candidate language, and
again queries $\verifier$ to obtain a new counterexample. This continues
until the learner can get a positive history bounded counterexample 
or there is no such counterexample. Thus, the learner now uses only
positive history bounded counterexamples, and hence, 
$\engine_\infcegis$ can identify any language that 
$\engine_\infhmncegis$ can identify. 
\comment{
If the memory of inductive learning engine is not finite, it can store all the examples
that are sent to it. So, the engine $\engine_\infcegis$ can use a verification oracle $\verifier$ providing
arbitrary counterexamples but internally the learner 
can record the history of positive examples and hence,
find the positive bounded counterexample through iterative calls
to the verifier. This can be done by checking whether the counterexample
is smaller than at least one of the positive examples seen so far. If it is not so, the counterexample
is removed from the proposed language, and the verification oracle is again queried by the learner to obtain a new
counterexample. This terminates when we obtain a counterexample smaller than the largest positive example
seen so far or we discover that no such counterexample exists. 
While the check is being done, the positive example \todo{transcript} can be stored using the infinite unbounded
memory. Thus, $\engine_\infcegis$ can simulate 
$\engine_\infhmncegis$ and hence, $\infhmncegis \subseteq \infcegis$.}

\end{enumerate}
}

We now present three languages used previously in proofs for inductive learning engines using 
\acta{finite} memory, and show how these languages allow us to distinguish relative power of synthesis engines.
\begin{enumerate}

\item Consider the language family~\ref{lf:cbweak}:\acta{ \
 $\mathcal{\lang}_{notcb} =\{ \lang_i |  
\todo{ i > \bound \text{ and } \lang_i = \{ n | n \in \nat \wedge n > i \}}
\}$ where
$\bound$ is a constant bound.} The argument in
Theorem~\ref{thm-b} also holds for the infinite memory synthesis engines, 
and so, $\mathcal{L}_{notcb} \in \overline \infbcegis \cap \infcegis$. 

Further, a positive history bounded verifier will always return a counterexample if one exists
since all counterexamples are smaller than any positive example in the language. Thus, $\engine_\infhmncegis$
can also identify languages in $\mathcal{L}_{notcb}$. 
Thus, $\mathcal{L}_{notcb} \in \overline \infbcegis \cap \infhmncegis$.

\item Consider the language family~\ref{lf:pbweak}: $\mathcal{L}_{notpb} = \{ \lang_i | i \in \nat \}$ where
$$\lang_i = \{ n | n \in \nat \wedge n \leq i \}$$
As argued in the proof of Theorem~\ref{thm-hb}, the verifier 
producing positive bounded counterexamples will not report any counterexample for any of the languages
in $\mathcal{L}_{notpb}$ because all counterexamples are larger than any positive example. So,
 languages in this family cannot be identified by $\engine_\infhmncegis$ but 
these can be identified using $\engine_\infcegis$.
So, $\mathcal{L}_{notpb} \in \overline \infhmncegis \cap \infcegis$.

\item Consider the finite language family~\ref{lf:cbnotpb}: $\mathcal{L}_{cbnotpb} = \{ \lang_i | i < \bound \}$ where
$$\lang_i = \{ n | n \in \nat \wedge n \leq i \}$$
As argued in proof of Theorem~\ref{thm-bhb}, the verifier  $\hmnverifier$ does not produce 
any counterexample for these languages since all positive examples are smaller than
any counterexample. But $\bverifier$ can be used to identify
languages in this class by selecting the bound to be $\bound$. Since, the counterexamples
are at most of size $\bound$ for these languages, a bounded counterexample verifier 
behaves exactly like an arbitrary counterexample producing verifier.
Thus,  $\mathcal{L}_{cbnotpb} \in \overline \infhmncegis \cap \infbcegis$. 
\end{enumerate}

We now summarize the results described in this section below.
For finite memory learners,
$\bcegis \subset \mncegis = \cegis$, 
$\hmncegis$ and $\cegis$ are not comparabale, that is,
$\hmncegis \not \subseteq \cegis$ and $\hmncegis \not \supseteq \cegis$.
$\bcegis$ and $\hmncegis$ are also not comparable.
In case of infinite memory learners,
$\infbcegis \subset \infmncegis = \infcegis$, and  
$\infhmncegis \subset \infcegis = \infmncegis$.
$\infbcegis$ and $\infhmncegis$ are again not comparable.
The results are summarized in Figure~\ref{fig:summary}.

\begin{figure}[htb!]
\centering
\includegraphics[height=7cm]{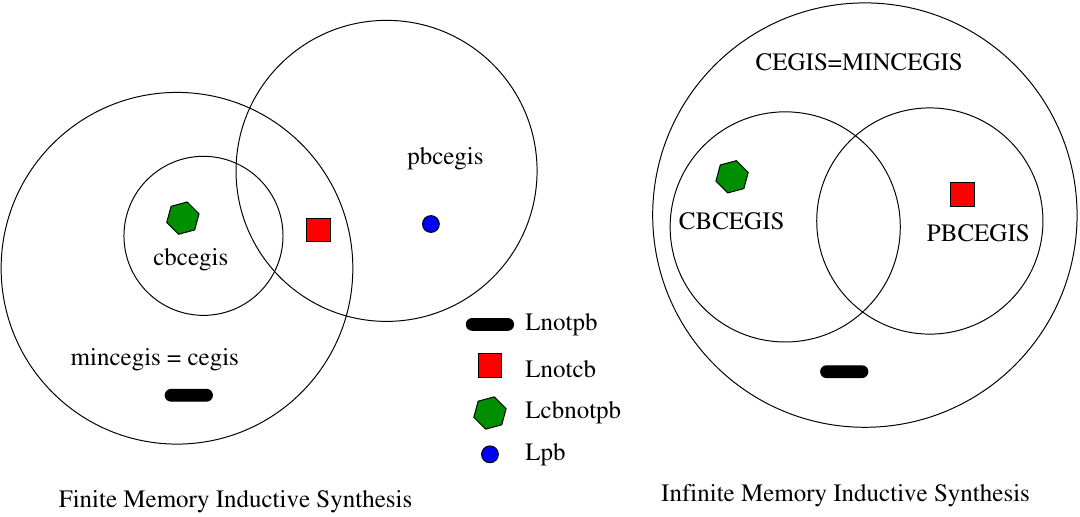}
\caption{Summary of Results on Decidability of Synthesis for Infinite Language Classes}
\label{fig:summary}
\end{figure}

\section{Analysis of $\ogis$ for Finite Language Classes}\label{sec-finres}

We now discuss the case when the class of candidate programs (languages) has finite cardinality.
As in Sec.~\ref{sec:not}, rather than referring to programs we will refer to
synthesizing the languages identified by those programs. 
If the language class is finite then
there exists a terminating $\ogis$ procedure, e.g., one that simply enumerates languages
from this class until one satisfying the specification $\Spec$ is obtained.
Moreover, any implementation of
$\ogis$ which uses an oracle that provides new (positive/negative) examples in every iteration ruling out at least one
candidate language will terminate with the correct language. The counterexample guided inductive
synthesis approach~\cite{asplos06} for bitvector sketches and oracle guided inductive synthesis using 
distinguishing inputs~\cite{jha-icse10} for programs composed of a finite library of components 
are examples of $\ogis$ synthesis techniques applied to finite language classes. We analyze the
complexity of synthesis for finite language classes and discuss its relation to the notion of {\it teaching dimension}
from the concept learning literature~\cite{goldman-td}.
This connection between synthesis of languages from finite classes and
teaching of concepts was first discussed in \cite{jha-icse10}. 
Here we establish that
the size of the smallest set of examples for language (program) synthesis is bounded below by the
teaching dimension of the concept class corresponding to the class of languages.

\subsection{NP-hardness}

We measure efficiency of an $\ogis$ synthesis engine using the notion of {\em sample complexity} 
mentioned in Sec.~\ref{sec:ogis} --- the number of queries 
(and responses) needed to correctly identify a language.
In order to analyze sample complexity, we need to fix the nature of queries to the oracle.
We focus on queries to which the oracle provides an 
example or counterexample in response.
We show that finding the minimal 
set of examples to be provided by the oracle such that the synthesis engine converges to 
the correct language is NP-hard.

\begin{theorem} 
Solving the formal inductive synthesis problem $\langle \class, \domain, \Spec, \orint \rangle$ for a finite $\class$ and finite $\domain$ 
with the minimum number of queries is NP-hard for any oracle interface $\orint$ comprising the correctness query $\qcorr$ 
(and possibly $qposwit$ and $qnegwit$).
%
\label{thm-finite}
\end{theorem}

\begin{proof}
\comment{
Consider a finite language family 
$$ \mathcal{L} = \{ L | L \subseteq \nat_k \}$$
where $\nat_k = \{1,2,\ldots,k\}$.
We consider the programs corresponding to the characteristic function for the languages, that is,
$$\algo_\lang(j) = 1 \;\text{ if } \; j \in \lang \; \text{ and }\; 0 \; \text{ otherwise}$$
Let the target program to be synthesized be $\lang^*$.
}
We prove NP-hardness through reduction from the minimum set cover problem.
Consider the minimum set cover problem with $k$ sets 
$S_1,S_2, \ldots, S_k$ and a universe comprising
$m$ elements $x_1,x_2, \ldots, x_m$ which needs to be covered
using the sets. We reduce it to a formal inductive synthesis problem 
$\langle \class, \domain, \Spec, \orint \rangle$ 
where $\class = \{ L_1, L_2, \ldots, L_m, L_{m+1} \}$ is a set of
$m+1$ languages, $\domain = \{ e_1, e_2, \ldots, e_k \}$ is the domain 
comprising $k$ examples over
which the languages are defined and $\Spec = \{ L_{m+1} \}$ is the specification.
Intuitively, 
the $m$ languages $L_1, \ldots, L_m$ are associated to the $m$ elements in the set cover
problem. 
The $k$ examples correspond to the $k$ sets. 
The sets $L_1, L_2, \ldots, L_{m+1}$ are constructed as follows:
For all $1 \leq i \leq k$ and $1 \leq j \leq m$,
example $e_i$ belongs to the symmetric difference of $L_j$ and $L_{m+1}$ 
if and only if 
the set $S_i$ contains element $x_j$.
We can do this, for instance, by including $e_i$ in $L_j$ but not in $L_{m+1}$.

Consider the operation of an $\ogis$ procedure implementing
an $\orint$ containing $\qcorr$.
Every unsuccessful correctness query returns a counterexample which is 
an element of $\domain$ in the symmetric difference of the proposed $L_j$ 
and $L_{m+1}$. 
Let $e_{i_1}, e_{i_2}, \ldots, e_{i_n}$ be the smallest
set of counterexamples that uniquely identifies the correct language $L_{m+1}$.
So, for all $ 1 \leq j \leq m$, there exists some $i_l$ such that 
either $e_{i_l} \in L_j$ or $e_{i_l} \in L_{m+1}$ but not both.
And so, for all $1 \leq j \leq m$, there exists some $i_l$ such that $x_j \in S_{i_l}$ where 
$i_l \in \{i_1, i_2, \ldots, i_n\}$. 
Moreover, dropping $i_l$ results in some $x_j$ not being covered (the corresponding $L_j$
is not distinguished from $L_{m+1}$).
Thus, $S_{i_1}, S_{i_2}, \ldots , S_{i_n}$ is a solution to the minimum set cover problem
which is known to be NP-complete. 
Similarly, it is easy to see that any solution to the minimum set cover problem
also yields a minimum counterexample set.

We can therefore conclude that 
solving the formal inductive synthesis problem $\langle \class, \domain, \Spec, \orint \rangle$ with
the minimum number of queries is NP-hard.

\qed
\end{proof}
We note that this proof applies to any FIS problem with an
oracle interface $\orint$ containing the correctness query $\qcorr$. 
Moreover, this proof can be easily extended to other oracle interfaces as well,
such as
the version of the distinguishing input method that does not use the correctness query,
with $\orint = \{ \qposwit, \qdiff, \qmem \}$. In this latter case, the combined use of
$\qdiff$ and $\qmem$ yields the desired mapping.

\subsection{Relation to Teaching Dimension}

Goldman et al.~\cite{goldman-td,goldman-brto} 
proposed {\it teaching dimension} as a measure to study
computational complexity of learning. They consider a teaching model in which a helpful teacher
selects the examples of the concept and provides it to the learner. Informally, the teaching
dimension of a concept class is the minimum number of examples that a teacher must reveal to 
uniquely
identify any target concept chosen from the class.

For a domain $\domain$ and concept class $\class$, 
a concept $\concept \in \class$ is a set of examples from $\domain$.
So, \acta{$\class \subseteq 2^{\domain}$}.
In the learning model proposed by Goldman et al.~\cite{goldman-td,goldman-brto}, 
the basic goal of the teacher is to help the learner identify the target 
concept $\concept^*\in \class$ by providing an example sequence from $\domain$.  
We now formally define the teaching dimension of a concept class.

\begin{definition} (adapted from~\cite{goldman-td})
An example sequence is a sequence of labeled examples from $\domain$,
\acta{where the labels are given by some underlying specification}. 
For concept class $\class$ and target concept
$\concept \in \class$, we say $T$ is a teaching sequence for $\concept$ 
(in $\class$) if $T$ is an example
sequence that uniquely identifies $\concept$ in $\class$ - 
that is, $\concept$ is the only concept in $\class$ consistent
with $T$. Let $T(\concept)$ denote the set of all teaching sequences for $\concept$. 
Teaching dimension $TD(\class)$ of the concept class is defined as follows:
$$ TD(\class) = \max_{\concept \in \class} \; (\min_{\tau \in T(\concept)} \, |\tau|) $$
\end{definition}

Consider an FIS problem where the specification is complete, i.e., $\Spec = \{ \lang_\target \}$.
Consider an instance of \ogis using any combination of witness, equivalence, subsumption, or
distinguishing input queries. Each of these queries, if it does not
terminate the \ogis loop, returns a new example for the learner. Thus,
the number of iterations of the \ogis loop, its sample complexity, is
the number of examples needed by the learner to identify a correct
language. Suppose the minimum such number of examples, for any
specification (target language $\lang_\target \in \class$), is $M_{\ogis}(\class)$. 
Then, the following theorem
must hold.
\begin{theorem}
$M_{\ogis}(\class) \geq TD(\class)$
\end{theorem}
The theorem can be obtained by a straightforward proof by
contradiction: if $M_{\ogis}(\class) < TD(\class)$, then 
for each target concept to be learned, 
there is a shorter teaching sequence than $TD(\class)$, viz., the one used by the
\ogis instance for that target, contradicting the definition of teaching dimension.

Now, given that the teaching dimension is a lower bound on the sample
complexity of \ogis, it is natural to ask how large $TD(\class)$ can
grow in practice. This is still a largely open question for general
language classes. However, results from machine learning theory can
help shed more light on this question.  
One of these results relates the teaching dimension to a second metric
for measuring complexity of learning, namely the {\it
Vapnik-Chervonenkis} (VC) dimension~\cite{vapnik-vc}. We define this below.

\begin{definition}~\cite{vapnik-vc}
Let $\domain$ be the domain of examples and $\concept$ be a concept from the class $\class$. 
A finite set $\domain' \subseteq  \domain$ is shattered
by $\class$ if $\{\concept \cap \domain' | \concept \in \class \} = 2^{\domain'}$. 
In other words, $\domain' \subseteq \domain$  is shattered by $\class$ if for 
each subset $\domain'' \subseteq \domain'$, 
there is a concept $\concept \in \class$ which contains all of $\domain''$, but none
of the instances in $\domain'-\domain''$. The Vapnik-Chervonenskis (VC) dimension is defined to be smallest $d$ 
for which no set of $d+1$ examples is shattered by $\class$.
\end{definition}

Blumer et al.~\cite{Blumer-vp} have shown that the VC dimension of a concept class
characterizes the number of examples required for learning any concept in the class under the 
distribution-free or probably approximately correct (PAC) model of Valiant~\cite{valiant:learnable}.
The differences between teaching dimension and 
Vapnik-Chervonenkis dimension are discussed at length by Goldman and Kearns~\cite{goldman-td}. 
The following theorems from ~\cite{goldman-td} provides 
lower and upper bound on the teaching dimension of a finite concept class in
terms of the size of the concept class and its VC-dimension.

\begin{theorem}~\cite{goldman-td}
The teaching dimension $TD(\class)$ of any concept class $\class$ satisfies
the following upper and lower bounds:
$$ VC(\class)/\log(|\class|) \leq TD(\class) \leq |\class| - 1$$
where $VC(\class)$ is the VC dimension of the concept class $\class$ and $|\class|$ denotes the number of concepts
in the concept class.
\end{theorem}
Moreover, Goldman and Kearns~\cite{goldman-td} exhibit a concept class
for which the upper bound is tight. This indicates that without
restrictions on the concept class, one may not be able to prove very strong bounds on the sample
complexity of \ogis.

\comment{
The following corollary follows from the equivalence between teaching dimension and the 
smallest number of examples needed to synthesize a program for a finite program class.

\begin{corollary}
Let $\mathscr{\algo}$ be any program class and $\mathscr{\lang}$ be the corresponding 
language class. Let $M = |\mathscr{\algo}| = |\mathscr{\lang}|$ be the number of candidate 
programs (or languages) in the class and let
the number of examples needed to synthesize program from the class be $k$, then
$$VC(\mathscr{\algo}) / \log (|\mathscr{\algo}|)  \leq k \leq |\mathscr{\algo}| - 1$$
\end{corollary}
} 

To summarize, we have shown that 
solving the formal inductive synthesis problem for finite domains and finite concept
classes with the minimum number of queries
is NP-hard. 
Further, we showed that the combinatorial measure of teaching
dimension captures the smallest number of examples required to identify 
the correct language.

\section{Conclusion} 
\label{sec:disco}
\label{sec-disco}
\label{sec:con}
\label{sec-con}

We  presented a theoretical framework and analysis of formal inductive 
synthesis by formalizing the notion of {\it oracle-guided inductive synthesis}
($\ogis$). We illustrated how $\ogis$ generalizes instances of concept 
learning in machine learning as well as synthesis
techniques developed using formal methods. 
We focus on counterexample-guided inductive synthesis (CEGIS) which is an
$\ogis$ implementations that uses 
the verification engine as the oracle. 
We presented different variations of $\cegis$ motivated by practice, and showed
that their synthesis power can be different, especially when the learning
engine can only store a bounded number of examples. 
%
There are several directions for future work. We  
discuss some open problems below that would further improve the theoretical
understanding of formal inductive synthesis.
\begin{myitemize}
\item  Teaching dimension of concept classes such as decision trees and
axis parallel rectangles have been well-studied in literature. But teaching
dimension of formal concept classes such as programs in the 
{\it while}~\cite{Winskel-book93}
language with only linear arithmetic over integers is not known.
Finding teaching dimensions for these classes would help in
establishing bounds
on the number of examples needed for synthesizing programs from these classes.
\comment { 
\item Complexity and decidability results for reactive synthesis are 
well-studied in literature. Recently, bounded synthesis techniques
have been proposed which bound some system parameter such as the number of 
states. The teaching dimension for the concept class of
universal co-Buchi automaton with bounded number of states would provide 
bounds on the number of examples needed for solving the bounded synthesis 
problem.
} 
\item We investigated the difference in synthesis power when the learning
engine has finite memory vs when the learning engine has infinite 
memory. Another important question to consider is how the power
of the synthesis engine changes when we restrict the time complexity of
learning engine such as the learning engines which take time polynomial in 
the number of examples. 

\item We have not analyzed the impact of different learning strategies that
may traverse the space of possible programs (languages) in various ways.
This is also an interesting avenue for future work.

\end{myitemize}

In summary, our paper is a first step towards a theory of formal inductive synthesis,
and much remains to be done to improve our understanding of this emerging area with
several practical applications.


\comment {
Hence, $\hmncegis$ can synthesize programs from some program classes where $\cegis$ fails to synthesize
the correct program. But contrariwise, $\hmncegis$ also fails at synthesizing programs from some program
classes where $\cegis$ can successfully synthesize a program. Thus, their synthesis power is not
equivalent, and none dominates the other.
}

\bibliographystyle{eptcs}
\bibliography{togis}

\end{document}